\documentclass[]{fairmeta}

\usepackage{microtype}
\usepackage{graphicx}
\usepackage{booktabs} 
\usepackage{adjustbox}

\usepackage{hyperref}
\usepackage{arydshln}

\usepackage{amsmath}
\usepackage{amssymb}
\usepackage{mathtools}
\usepackage{amsthm}
\usepackage{multicol}
\usepackage{multirow}
\usepackage{resizegather}

\theoremstyle{plain}
\newtheorem{theorem}{Theorem}[section]

\newtheorem{lemma}[theorem]{Lemma}

\theoremstyle{definition}

\theoremstyle{remark}

\usepackage[textsize=tiny]{todonotes}

\newcommand{\benchmarkname}{ColBench}
\newcommand{\methodname}{SWEET-RL}

\newcommand{\EE}{\mathbb{E}}

\title{\methodname{}: Training Multi-Turn LLM Agents on Collaborative Reasoning Tasks}

\author[1,2,\dagger]{Yifei Zhou}
\author[1]{Song Jiang}
\author[1]{Yuandong Tian}
\author[1]{Jason Weston}
\author[2]{Sergey Levine}
\author[1,*]{Sainbayar Sukhbaatar}
\author[1,*]{Xian Li}

\affiliation[1]{FAIR at Meta}
\affiliation[2]{UC Berkeley}

\contribution[\dagger]{Work done at Meta}
\contribution[*]{Equal advising}

\abstract{
Large language model (LLM) agents need to perform multi-turn interactions in real-world tasks. However, existing multi-turn RL algorithms for optimizing LLM agents fail to perform effective credit assignment over multiple turns while leveraging the generalization capabilities of LLMs and it remains unclear how to develop such algorithms. To study this, we first introduce a new benchmark, \benchmarkname{}, where an LLM agent interacts with a human collaborator over multiple turns to solve realistic tasks in backend programming and frontend design. Building on this benchmark, we propose a novel RL algorithm, \methodname{} (RL with Step-WisE Evaluation from Training-time information), that uses a carefully designed optimization objective to train a critic model with access to additional training-time information. The critic provides step-level rewards for improving the policy model. Our experiments demonstrate that \methodname{} achieves a 6\% absolute improvement in success and win rates on \benchmarkname{} compared to other state-of-the-art multi-turn RL algorithms, enabling Llama-3.1-8B to match or exceed the performance of GPT4-o in realistic collaborative content creation.
}

\date{\today}
\correspondence{Yifei Zhou at \email{yifei\_zhou@berkeley}}
\code{\href{https://github.com/facebookresearch/sweet_rl}{https://github.com/facebookresearch/sweet\_rl}}
\data{\href{https://huggingface.co/datasets/facebook/collaborative_agent_bench}{https://huggingface.co/datasets/facebook/collaborative\_agent\_bench}}

\begin{document}
\maketitle

\vspace{-0.1cm}
\section{Introduction}
\vspace{-0.1cm}

\begin{figure*}[!h]
     \centering
     \vspace{-0.3cm}
    \includegraphics[width=0.96\textwidth]{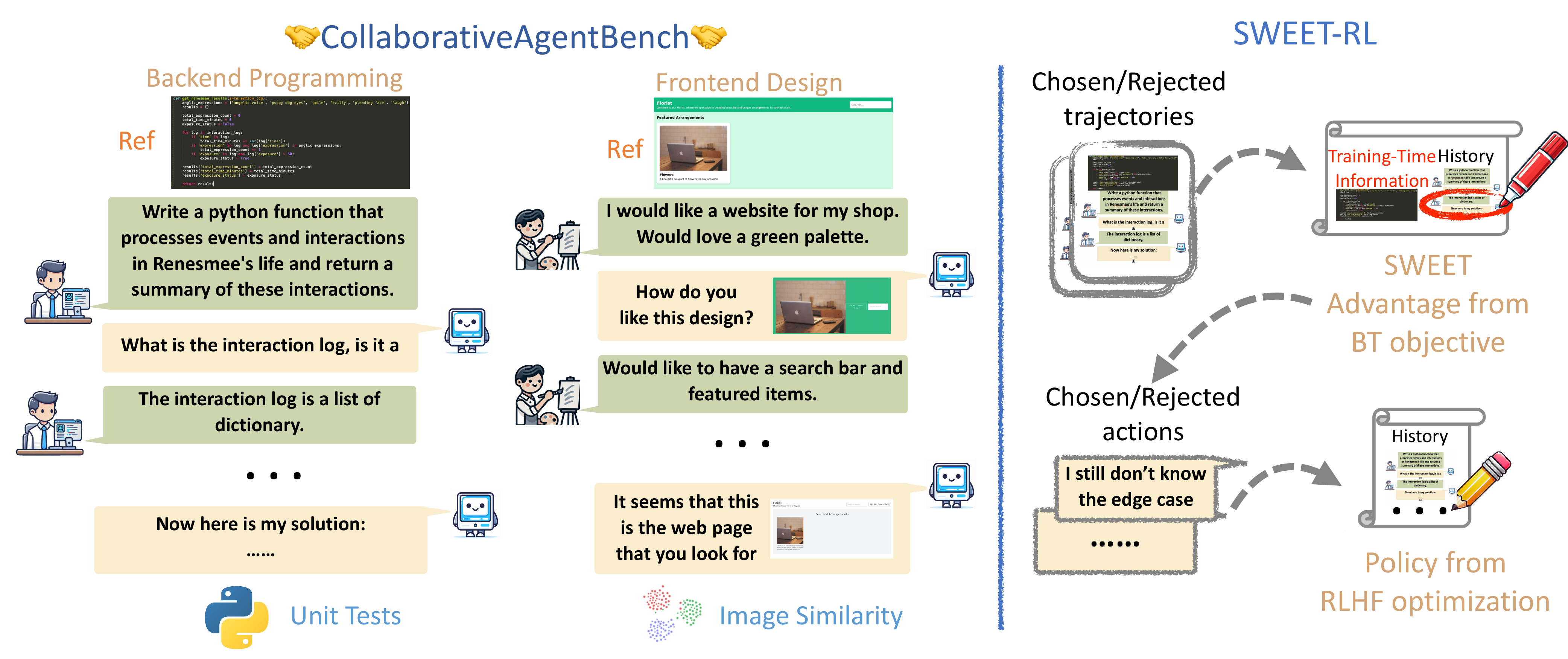}
     \vspace{-0.2cm}
        \caption{\emph{(Left)} \textbf{Overview of our \benchmarkname{} including Backend Programming and Frontend Design} that supports cheap and reliable evaluation of multi-turn RL algorithms for agents in realistic settings. \emph{(Right)} \textbf{The high-level motivation behind \methodname{}} that uses additional training-time information along with appropriate Bradley-Terry (BT) objective to perform effective credit assignment.}
        \label{fig:teaser} 
        \vspace{-0.2cm}
\end{figure*}

Large language models (LLMs) have the potential of serving as decision-making agents capable of executing complex tasks autonomously, such as navigating the web and controling devices~\citep{zhou2024webarenarealisticwebenvironment, zhou2024proposeragentevaluatorpaeautonomousskilldiscovery, xu2024aguvisunifiedpurevision, gur2023webagent, bai2024digirltraininginthewilddevicecontrol}, writing and maintaining code bases~\citep{jimenez2024swebenchlanguagemodelsresolve}, and serving as personal assistants~\citep{xie2024travelplannerbenchmarkrealworldplanning, jiang2024delegationdesigningidealagentic}, thanks to advances in the reasoning and generalization capabilities of LLMs~\citep{openai2024gpt4technicalreport, geminiteam2024geminifamilyhighlycapable,dubey2024llama3herdmodels}. However, to achieve the best performance on tasks that involve making a sequence of decisions, the agent needs to directly optimize for the multi-turn objective of interest such as success rates, which is more challenging than only imitating the most probable action at each turn as learnt in the next-token-prediction pre-training objective.

While a natural approach to directly optimize for multi-turn objective is to apply successful algorithms from single-turn Reinforcement Learning from Human Feedback (RLHF)~\citep{ouyang2022traininglanguagemodelsfollow, ziegler2020finetuninglanguagemodelshuman, christiano2023deepreinforcementlearninghuman}, such as RAFT~\citep{dong2023raftrewardrankedfinetuning}, DPO~\citep{rafailov2024directpreferenceoptimizationlanguage} and PPO~\citep{schulman2017proximalpolicyoptimizationalgorithms}, these methods do not perform explicit credit assignment across turns.
Consequently, they may suffer from high variance and poor sample complexity due to the long-horizon nature of complex sequential decision-making tasks~\citep{zhou2024archertraininglanguagemodel}.
Another alternative is to apply value function learning methods, such as TD-learning~\citep{mnih2013playingatarideepreinforcement, zhou2024archertraininglanguagemodel, snell2023offlinerlnaturallanguage}. Yet that requires training a new task-specific value head on top of LLM representations, which may not generalize well with limited fine-tuning data (\autoref{fig:scaling}). As a result, it is still unclear what is the most effective multi-turn RL algorithm that fully takes advantage of the reasoning  capabilities of LLMs for training general, capable, and goal-directed agents.

To begin to address this challenge, we first notice that a benchmark for validating multi-turn RL algorithms for realistic LLM agents is required to meet the following three criteria: 1) sufficient task diversity for RL training without overfitting, 2) sufficient task complexity that challenges the reasoning and generalization capability of agents, and 3) minimum engineering overhead for fast research prototyping. However, as shown in \autoref{tab:benchmarks}, none of the existing benchmarks satisfy all of the three necessary criteria.
To address this gap, \textbf{our first contribution in this work} is to develop a benchmark, \textbf{Collaborative Agent Benchmark (\benchmarkname{})}, designed to support research on multi-turn RL algorithms for realistic LLM agent scenarios. Our benchmark focuses on the realistic domain of artifact creation, where the goal for agents is to interact with humans to produce a final artifact (e.g., code, web pages, or slides) that fulfills human expectations. To solve such tasks, the agent must act to understand the intent of the human collaborator and reason with the missing pieces, as only limited information is provided in each turn for complex and potentially multi-modal artifacts like code and web pages. To facilitate rapid iteration and cost-effective evaluation, we employ LLMs as human ``simulators'', where we crucially also provide the ground-truth artifacts to ensure faithful simulations in their responses. For reliable evaluations, we have developed a series of functional evaluators that measure the similarity between the agent-produced artifact and the ground-truth. Examples of tasks in \benchmarkname{} are shown in \autoref{fig:teaser}(left) and full trajectories in \autoref{app:full_qualitative}. 

For a multi-turn RL algorithm to perform effective credit assignments in such LLM agent settings, it needs to incorporate solutions to the following realistic challenges. Firstly, the agent is acting in a partially observable environment where some of the important task-relevant information is not directly revealed to the agent. In such cases, the agent needs to be properly awarded for information-seeking behaviors in the stochastic environment. Moreover, there is only limited amount of data available during fine-tuning for a large diverse set of tasks that may show up at test time. Therefore, the learning objective of the algorithm needs to effectively take advantage of the reasoning capabilities of LLMs for the best generalization performance. For the first challenge, we observe that additional training-time information, such as the final outcome and the reference solution, may be available during training. The knowledge of additional training-time information can offer a shortcut for credit assignments for an agent performing information-seeking behaviors without such knowledge. A natural approach to leveraging this training-time information is to train a value function that predicts the expected utility of each action as a scalar value. However, this introduces a fine-tuning objective that significantly differs from the next-token prediction pre-training objective of LLMs, which leads to inferior reasoning and generalization performance (\autoref{fig:scaling}).

\begin{table}[ht]
    \centering
    \small
    \setlength{\tabcolsep}{2.0pt}
        \caption{\textbf{Comparisons between \benchmarkname{} and existing benchmarks for multi-turn LLM agents.} As shown in the table, no existing benchmarks satisfy all of the three criterions necessary for developing efficient RL algorithms for fine-tuning LLM agents: 1) sufficient task diversity for RL training without overfitting, 2) sufficient task complexity that challenges the reasoning and generalization capability of agents, and 3) minimum engineering overhead for fast research prototyping. }
  \begin{adjustbox}{max width=\linewidth}
        \begin{tabular}{lccc}
            \toprule
            & {RL} & Complex & {Min}\\ 
            & {Training} & {Reasoning} & { Overhead}\\ 
            \midrule
             \textbf{Web/Device-Control Agents} & \multirow{2}{*}{Yes}  & \multirow{2}{*}{Yes} & \multirow{2}{*}{\textbf{No}}\\
             \citep{zhou2024webarenarealisticwebenvironment, xie2024osworldbenchmarkingmultimodalagents}& & & \\
             \textbf{SWEBench} & \multirow{2}{*}{Yes}  & \multirow{2}{*}{Yes} & \multirow{2}{*}{\textbf{No}}\\
             \citep{jimenez2024swebenchlanguagemodelsresolve, pan2024trainingsoftwareengineeringagents}& & & \\
             \textbf{Travel Planner} \citep{xie2024travelplannerbenchmarkrealworldplanning} & {\textbf{No}}  & {Yes} & {Yes}\\
             \textbf{LLF Benchmark} \citep{cheng2023llfbenchbenchmarkinteractivelearning} & {\textbf{No}}  & {\textbf{No}} & {Yes}\\
             \textbf{AgentBench} \citep{liu2023agentbenchevaluatingllmsagents} & {\textbf{No}}  & {Yes} & {Yes}\\
             \textbf{Mint} \citep{wang2024mintevaluatingllmsmultiturn} & {\textbf{No}}  & {Yes} & {Yes}\\
             \textbf{Dialop} \citep{lin2024decisionorienteddialoguehumanaicollaboration} & {\textbf{No}}  & {Yes} & {Yes}\\
             \textbf{LMRL Gym} \citep{abdulhai2023lmrlgymbenchmarksmultiturn} & {Yes}  & {\textbf{No}} & {Yes}\\
             \textbf{RL4VLM} \citep{zhai2024finetuninglargevisionlanguagemodels} & {Yes}  & {\textbf{No}} & {Yes}\\
             \midrule
             \textbf{\benchmarkname{}} (ours)  & {Yes}  & {Yes} & {Yes}\\
            \bottomrule
        \end{tabular}
    \end{adjustbox}
        \vspace{-2mm}
        \label{tab:benchmarks}
        \vspace{-0.3cm}
\end{table}

With these observations, \textbf{the second contribution of this work} is an easy-to-implement yet highly effective RL algorithm, \textbf{\methodname{}} (RL with \textbf{S}tep-\textbf{W}is\textbf{E} \textbf{E}valuation from \textbf{T}raining-Time Information) as depicted in \autoref{fig:teaser}(right). \methodname{} improves credit assignments by providing the critic with training-time information that is inaccessible to the actor.
Our novel turn-level critic takes advantage of this asymmetric observation spaces for the critic and actor. Furthermore, we propose directly learning the advantage function, which characterizes the effectiveness of each action at the current state, avoiding the need of first training a value function that predicts the expected utility of the current state and action. Finally, we also propose parameterizing the advantage function by the mean log probability of the action at each turn and training this advantage function through the Bradley-Terry objective at the trajectory level. We find such an objective aligns better with pre-trained LLMs compared to the common practice of training a value head on top of the hidden states of LLMs, leading to superior generalization results. In our experiments, we find that the use of asymmetric information during training and appropriate learning objectives result in a superior multi-turn agent on both realistic Backend Programming and Frontend Design tasks from \benchmarkname{}, with 6\% absolute success and win rates gains compared to other SOTA algorithms. As a result, the performance of Llama-3.1-8B~\citep{dubey2024llama3herdmodels} can match or even surpass the performance of SOTA proprietary models including GPT-4o and o1-mini~\citep{openai2024gpt4technicalreport}.

\vspace{-0.1cm}
\section{Related Work}
\vspace{-0.1cm}

\textbf{Benchmarks for LLM Agents.}
While many recent benchmarks have been proposed to evaluate the capabilities of LLM agents in various settings, such as software engineering~\citep{jimenez2024swebenchlanguagemodelsresolve, liu2023agentbenchevaluatingllmsagents}, web navigation~\citep{zhou2024webarenarealisticwebenvironment, koh2024visualwebarenaevaluatingmultimodalagents, deng2023mind2webgeneralistagentweb, yao2023webshopscalablerealworldweb}, device control~\citep{rawles2023androidwildlargescaledataset, rawles2024androidworlddynamicbenchmarkingenvironment, xie2024osworldbenchmarkingmultimodalagents}, and travel planning~\citep{xie2024travelplannerbenchmarkrealworldplanning}, most of them tend to focus on evaluation of state-of-the-art generalist LLMs without providing a research-friendly interactive environment and a set of training tasks to study multi-turn RL algorithm. While LMRL Gym~\citep{abdulhai2023lmrlgymbenchmarksmultiturn} and RL4VLM~\citep{zhai2024finetuninglargevisionlanguagemodels} offers this flexibility for comparing different multi-turn RL algorithms, the task settings focus on narrower domains and do not require the model to have strong reasoning capabilities. As shown in \autoref{tab:benchmarks}, there is no existing LLM agent benchmark that provides the flexibility for testing multi-turn RL algorithms on reasoning-intensive tasks with minimum engineering overhead. In contrast, \benchmarkname{} is the first benchmark designed to support research efforts in multi-turn RL algorithms on reasoning-intensive tasks, focusing on the realistic domain of artifact creation with reliable functional verifiers.

\textbf{Multi-turn RL algorithms for LLM Agents.}
Unlike single-turn scenarios such as single-turn preference optimization~\citep{christiano2023deepreinforcementlearninghuman, ziegler2020finetuninglanguagemodelshuman, casper2023openproblemsfundamentallimitations,xu2023some} where it suffices for LLMs to produce a single response without further interactions with the environment, multi-turn RL~\citep{zhou2024archertraininglanguagemodel, kumar2024traininglanguagemodelsselfcorrect, abdulhai2023lmrlgymbenchmarksmultiturn} captures realistic agent scenarios where LLM agents need to make a sequence of actions to complete the task, such as operating a unix terminal~\citep{liu2023agentbenchevaluatingllmsagents} and navigating through the web~\citep{zhou2024webarenarealisticwebenvironment}. While some early works directly applied successful methods from single-turn RL, such as REINFORCE~\citep{REINFORCE, wu2018learningextractcoherentsummary}, DPO~\citep{xiong2024buildingmathagentsmultiturn, song2024trialerrorexplorationbasedtrajectory}, and PPO~\citep{schulman2017proximalpolicyoptimizationalgorithms, szot2024largelanguagemodelsgeneralizable}, they often suffer from high variance when the horizon gets longer, resulting in poor 
performance. While recent works have applied more advanced techniques from the deep RL literature such as Bellman bootstrapping~\citep{zhou2024archertraininglanguagemodel, snell2023offlinerlnaturallanguage} and Path Consistency~\citep{wang2024offlinereinforcementlearningllm, liu2024enhancingmultistepreasoningabilities, nachum2017bridginggapvaluepolicy} to reduce long-horizon variance, our work makes an important advancement to take advantage of the oft-neglected additional training-time information and make corresponding adjustments to the optimization objective for improved assignment. Finally, while some prior works apply an asymmetric actor-critic structure to perform sim-to-real transfer in robotics where the critic observes the latent state and the actor observes RGB inputs~\citep{pinto2017asymmetricactorcriticimagebased, wilson2020learningmanipulateobjectcollections, Salter2019AttentionPR}, less has been studied in terms of how such techniques can be applied in reasoning-intensive LLM tasks.


\textbf{Process reward models.} The use of a step-wise critic resembles the notion of a ``process reward model'' (PRM) in the reasoning literature~\citep{lightman2023letsverifystepstep, uesato2022solvingmathwordproblems}. PRMs evaluate the ``correctness'' of each reasoning step and can be trained from automated supervision~\citep{luo2024improvemathematicalreasoninglanguage, setlur2024rewardingprogressscalingautomated, yuan2024freeprocessrewardsprocess, hwang2024selfexploreenhancingmathematicalreasoning} without costly human-annotated process labels. Once a PRM is trained, it can be used for searching with more test-time compute~\citep{snell2024scalingllmtesttimecompute, yuan2024freeprocessrewardsprocess} or accelerating the explorations in on-policy RL~\citep{setlur2024rewardingprogressscalingautomated, shao2024deepseekmathpushinglimitsmathematical,lin2025step} (i.e. when the policy is trained on the online collected trajectories by itself). In contrast, in our work, the step-wise critic is mainly used to perform credit assignment as an intermediate ``reward proxy'' to directly optimize the policy without the need to collect additional interaction data. This benefit is  important in LLM agent tasks where collecting on-policy data involves expensive interactions with an external environment.



\vspace{-0.1cm}
\section{Collaborative Agent Benchmark (\benchmarkname{})}
\vspace{-0.1cm}
In this section, we first outline the foundational design principles of our \benchmarkname{} to address challenges in developing multi-turn RL algorithms for LLM agents, followed by a detailed explanation of the two specific tasks: 
Backend Programming and Frontend Design.

\vspace{-0.1cm}
\subsection{Design Principles}
\vspace{-0.1cm}
\textbf{(1) Sufficient task complexity that challenges the rea-
soning and generalization capability of agents.} As the ultimate goal is to enable the LLM agents to complete tasks on our behalf in the real world, where the agents need to address complicated challenges in out-of-distribution scenarios, it is essential that the benchmark reflects such realistic reasoning and generalization challenges. To achieve this, \benchmarkname{} is designed to align with realistic artifact creation tasks, where the objective of collaboration is to produce tangible outputs such as code or web pages. In order to do well, it is necessary for agents to dive deep into the structure of the code and nuanced differences in visual design for unseen requests and a potentially sub-optimal collaborator. 


\textbf{(2) Minimum 
overhead for fast research prototyping.} To achieve this, we ground each collaboration task with the goal to re-create the exact same product as the reference artifact. In this way, the human collaborator can be easily simulated by an LLM with access to the reference artifact to faithfully answer the uncertainties from the LLM agent. Furthermore, the presence of a reference artifact allows \benchmarkname{} to employ objective and functional evaluation metrics that assess the similarity between the final collaborative product and the reference artifact. As a result, the only requirement needed for setting up \benchmarkname{} is API access for some LLM calls and some Python packages for running code and rendering HTML web pages. Although in real-world scenarios the human collaborator might only have a general idea of the desired final product, having access to a clear reference artifact is a reasonable assumption that significantly enhances reliability.

\textbf{(3) Sufficient task diversity for RL training without
overfitting.} As the best LLMs are often trained on a huge amount of data, it is essential that our benchmark contains enough tasks to understand the scalability of different multi-turn RL algorithms while at the same time ensuring the reliability of simulation and evaluation. Therefore, we designed \benchmarkname{} to be highly scalable, with more than 10k different procedurally-generated tasks, which also can be easily expanded as needed by incorporating more existing artifacts such as code and web pages. The difficulty of \benchmarkname{} can be easily adjusted 
by simply creating more collaboration tasks with more complicated code bases and web pages.

Next, we will describe detailed setups of \benchmarkname{}, including Backend Programming and Frontend Design.
\vspace{-0.1cm}
\subsection{Backend Programming Collaborations}
\vspace{-0.1cm}

\textbf{Task description.} In this task, the agent is required to collaborate with the human simulator to write a custom Python function (up to 50 lines). In the beginning of the collaboration, the agent is provided with a high-level description and the signature of the function. However, many concrete details, such as what conditions should be considered and what to do at edge cases, are not provided, and the agent has to reason and decide what clarifications are needed from the human simulator. The human simulator will provide a brief explanation in natural language to each clarification question based on the reference code visible only to the human simulator, but it will not write code. The interactions are limited to 10 back-and-forth rounds between the agent and the human simulator. The interaction ends either when the agent has decided to give the final solution or the maximum number of rounds has been reached. The success of the agent is evaluated by 10 hidden unit tests for each function for 0/1 reward only at the end of each collaboration.

\textbf{Data generation.} Python functions, high-level descriptions, and unit tests are generated by prompting Llama-3.1-70B-Instruct~\citep{dubey2024llama3herdmodels} to extract a python function as inspired by an Internet excerpt from DCLM~\citep{li2024datacomplmsearchgenerationtraining}. Only the tasks where the generated python functions can pass their corresponding unit tests are kept. We generate 10k such tasks in the train set and 1k tasks in the test set where tasks in the test set are manually inspected by authors. 15k offline train trajectories are generated by zero-shot prompting Llama-3.1-8B-Instruct as agent and Llama-3.1-70B-Instruct as human simulator.

\vspace{-0.1cm}
\subsection{Frontend Design Collaborations}
\vspace{-0.1cm}

\textbf{Task description.} In this task, the agent is required to collaborate with the human simulator to design a web page by writing an HTML snippet (around 100 lines). At the beginning of the collaboration, the agent is provided with a high-level description of the web page. Again, many concrete details such as the layout and color palette of the web page are missing and only visible to the human simulator. At each round, the agent has a chance to write an HTML solution and it will be rendered by the web browser. The human simulator will be able to examine the rendered web page from the agent and the reference web page, then describe their differences to the agent. Similar to the backend programming collaboration, the interaction ends either when the agent has decided to give the final solution or the maximum number of 10 rounds has been reached. The performance of the agent is evaluated by the cosine similarity of CLIP~\citep{radford2021learningtransferablevisualmodels} embeddings between the agent solution and reference web page, the best metric found in prior works~\citep{si2024design2codefarautomatingfrontend} for frontend design. It serves as a reward within the range of 0 to 1 only at the end of the collaboration.

\textbf{Data generation.} The tasks containing reference web pages and high-level descriptions are from WebSight~\citep{laurençon2024unlockingconversionwebscreenshots}. We include 10k such tasks for training and 500 for the test set. The test set tasks are manually inspected by the authors. We generate 6k offline train trajectories by zero-shot prompting Llama-3.1-8B as agent and Qwen2-VL-72B~\citep{yang2024qwen2technicalreport} as human simulator.


\vspace{-0.1cm}
\section{\methodname{}: RL with \textbf{S}tep-\textbf{W}is\textbf{E} \textbf{E}valuation from \textbf{T}raining-Time Information} \label{sec:method}
\vspace{-0.1cm}
To introduce our method, we will begin by first defining the terminology for multi-turn RL on \benchmarkname{}. Then, we will motivate the two-stage training procedure of \methodname{} for first training a step-wise critic with additional training-time information and using it as a per-step reward model to train the actor with careful algorithmic choices. An overview of the two-stage training procedure is presented in \autoref{fig:method}.


\begin{figure*}[!ht]
     \centering
     \vspace{-0.3cm}
    \includegraphics[width=0.96\textwidth]{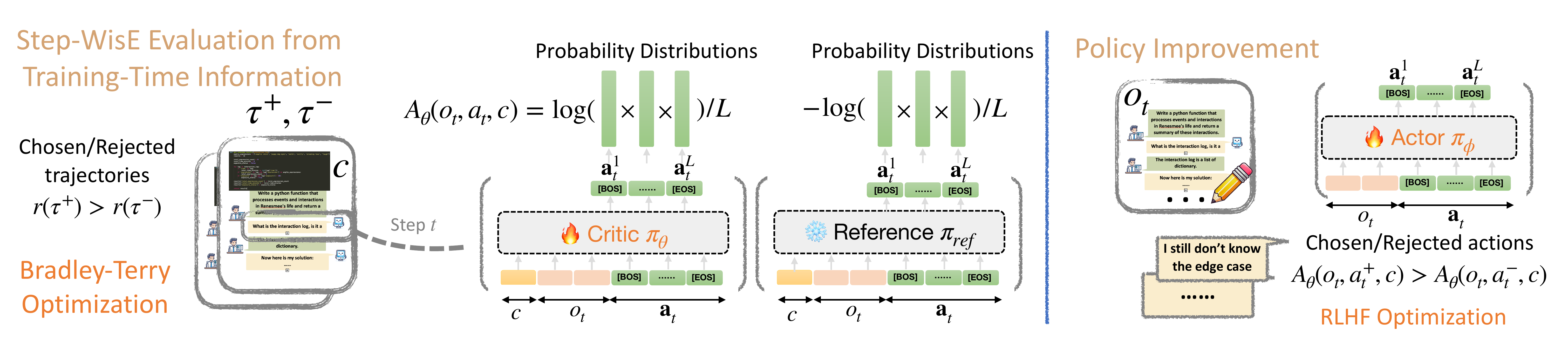}
     \vspace{-0.3cm}
        \caption{textbf{An overview of the training procedure of \methodname{}.} At a high level, we first apply Bradley-Terry objective to directly train a step-wise advantage function with access to additional training-time information. Once the advantage function is trained, we perform policy improvement by using the advantage function as a reward model for each turn.}
        \label{fig:method} 
        \vspace{-0.5cm}
\end{figure*}
\vspace{-0.1cm}
\subsection{Problem Setup} \label{sec:preliminaries}
\vspace{-0.1cm}

We frame the problem of collaboration between humans and agents as a finite-horizon partially observable Markov decision process (POMDP) $\mathcal{M} = \{\mathcal{O}, \mathcal{C},\mathcal{A}, \mathcal{T}, \mu_1, \mathcal{R}, N\}$. 
Here, $\mathcal{O}$ and $\mathcal{C}$ are the observable and hidden  parts of the state space. In the beginning of each episode an initial instruction $o_1 \in \mathcal{O}$ and some hidden training-time information $c \in \mathcal{C}$ (e.g., the reference solution) are drawn from the initial state distribution $\mu_1$. The hidden training-time information $c$ remains unchanged during the episode.

At the $t$-th turn, the observation $o_t$ for the agent includes all the interaction history, and the agent needs to take an action $a_t \in \mathcal{A}$ by outputting a response consisting of tokens $a_t^{1:L}$. After taking an action, the user will respond to the agent, and the new state for the agent (represented by the transition function $\mathcal{T}$) is derived by appending the latest interaction to the interaction history. At each step, the agent has a chance to receive a scalar reward $r(o_t, a_t, c) \in \mathcal{R}$. The episode ends either when the agent outputs a termination token or when a maximum number of $t=N$ rounds of interactions has been reached.
The objective of RL is to train a policy that can generate token sequences maximizing the cumulative rewards 
$\sum_{t=1}^N r(o_t, a_t, c)$
throughout a rollout (we assume no reward decay for simplicity).
We consider the offline setting of learning from a past interaction dataset as online interactions with humans may be costly to obtain.

The Q-function for a policy \( \pi \) represents the expected cumulative reward of a specific action at the current step, followed by adherence to policy \( \pi \): 
$Q^\pi(o_t,a_t, c) = \mathbb{E}_{\pi} \left[\sum_{t'=t}^N r(o_{t'}, a_{t'}, c)\right].$
The value function of \( \pi \), \( V^\pi(o_t, c) \), is defined as the expected Q-value, \( \mathbb{E}_{a_t \sim \pi} [Q^\pi(o_t, a_t, c)] \), where actions \( a_t \) are sampled from \( \pi \). The advantage function \( A^\pi(o_t,a_t, c) \) indicates the relative benefit of taking action \( a_t \) in state \( (o_t,c) \) and is calculated as the difference between the Q-value and the state's value under the policy: $A^\pi(o_t,a_t,c) = Q^\pi(o_t, a_t,c) - V^\pi(o_t,c)$. We directly model the turn-wise advantage function \( A \) using parameters \( \theta \), and use that advantage $A_\theta$ to train the policy parameterized by \( \phi \) to generate tokens $a_t^{1:L}$ within each turn.

\vspace{-0.1cm}
\subsection{Learning turn-wise advantage functions}\label{sec:credit_assignment}
\vspace{-0.1cm}


To perform explicit credit assignments in reasoning-intensive tasks,
some prior works have explored learning an explicit value function first and derive the advantage of each individual action from the learnt value function~\citep{bai2024digirltraininginthewilddevicecontrol, zhou2024archertraininglanguagemodel, snell2023offlinerlnaturallanguage}. However, in our experiments we found that such value functions do not generalize well when only a limited number of samples for fine-tuning are available as shown in \autoref{sec:discussions}. 
We hypothesize this is because learning an accurate value function in reasoning-intensive tasks is itself a hard task and does not effectively take advantage of the reasoning and generalization capability of pre-trained LLMs. 


Since the ultimate goal of performing credit assignment is to derive the advantages for each action which may be a easier task for LLMs compared to estimating the expected future returns, we propose to directly learn the advantage function for each turn-wise action $a_t$. Inspired by the success of preference optimization in funetuning LLMs~\citep{christiano2023deepreinforcementlearninghuman, ziegler2020finetuninglanguagemodelshuman}, we propose to train the the turn-wise advantage function from preference pairs of trajectories.
Given two trajectories under the same task with additional training-time information $c$, we label them chosen $\tau^+$ and rejected $\tau^-$ judged by their cumulative rewards. This allows us employ the Bradley-Terry objective~\citep{Bradley1952RankAO,rafailov2024rqlanguagemodel} for fine-tuning 
\begin{equation}
        \mathcal{J}_\text{BT} =-\log \left[\sigma\left(\sum_{t} \beta r(o^+_t, a^+_t, c) - \sum_{t} \beta r(o^-_t, a^-_t, c)\right)\right], \label{eqn:BT}
\end{equation}
where $o^+_t, a^+_t $ and $o^-_t, a^-_t$ are from $\tau^+$ and $\tau^-$ respectively, and $\beta$ is a hyperparameter.
We can rewrite this objective using the advantage function:
\begin{equation}
        \mathcal{J}_{A}(\theta) = -\log \left[\sigma\left(\sum_t \beta A_\theta(o^+_t, a^+_t, c) - \sum_{t} \beta A_\theta(o^-_t, a^-_t, c)\right)\right].\label{eqn:step_advantage}
\end{equation}
Intuitively, similar to the objective of single-turn RLHF to learn a high reward for each chosen response and a low reward for each rejected response, the effect of \autoref{eqn:step_advantage} is to increase the advantage for each action in the chosen trajectory and lower the advantage for each action in the rejected trajectory. For completeness, we provide a theoretical derivation in \autoref{app:theory}. To further align the learning objective with next-token-prediction pre-training, we parameterize the advantage function by re-purposing the existing language model head of the LLM:
\begin{equation}
    A_\theta (o_t, a_t, h) = \frac{1}{L} \sum_{l=1}^L \left[ \log \frac{ \pi_\theta (a^l_t|o_t, a_t^{1:l-1}, c) }{ \pi_\text{ref} (a^l_t|o_t, a_t^{1:l-1}, c) } \right] \label{eqn:advantage_parameterization}
\end{equation}
where $\pi_\theta$ is the LLM model that we train to act as the advantage function, while $\pi_\text{ref}$ is a frozen initial seed model. 
We find it important to include $\frac{1}{L}$ to normalize the advantage by the length of the response to stabilize training.
\vspace{-0.1cm}
\subsection{Optimizing the agent by turn-wise advantage}
\vspace{-0.1cm}


Our key observation is that while our final policy $\pi_\phi$ cannot condition on the hidden information $h$, such information is available during training time.
Since the advantage LLM $\pi_\theta$ will  only be used during training, it can take $c$ as input as in \autoref{eqn:advantage_parameterization}.
Intuitively, many realistic problems such as collaboration and math reasoning have some hidden training-time information like reference solutions. 
If the turn-wise advantage function has access to such training-time information, it should be in a better position to judge whether the action taken by the policy is on the right track. 

Therefore, we provide additional training-time information $c$ to the turn-wise advantage function while only the interaction history $o_t$ is provided to the policy, resulting in an asymmetric actor-critic structure. In principle, any successful algorithm from the RLHF literature can be used to optimize the per-turn policy $\pi_\phi$ by treating the interaction histories as prompts and the turn-wise advantage function $A_\theta$ as the reward model. In this stage of training the policy, no interaction from human collaborators is needed. 

For simplicity we choose to use DPO~\citep{rafailov2024directpreferenceoptimizationlanguage} for training. For each turn $t$ we first sample candidate actions from the current policy given  interaction history $o_t$, and rank them by the learnt turn-wise advantage function to obtain chosen and rejected actions. 
We then optimize the policy for each turn 
using the standard DPO loss:

\vspace{-5mm}
\begin{equation}
        \mathcal{J}_\pi (\phi) = 
        -\log \sigma \left( \beta' \frac{\log \pi_\phi(a^+|o_t)}{\log \pi_\text{ref}(a^+|o_t)} - \beta' \frac{\log \pi_\phi(a^-|o_t)}{\log \pi_\text{ref}(a^-|o_t)} \right). \label{eqn:DPO}
\end{equation}
In practice, for each turn we sample $16$ candidate actions and take random actions from top-50\% quantile as chosen and from the bottom-50\% quantile as rejected.

\vspace{-0.1cm}
\section{Experiments} \label{sec:experiments}
\vspace{-0.1cm}
The purpose of our experiments is to validate 
the effectiveness of \methodname{} as a multi-turn RL algorithm  that trains LLM agents 
for complex collaborative tasks. Specifically, they are designed to answer the following questions: \textbf{(1)} How do existing generalist models and multi-turn RL algorithms perform for collaborative tasks on \benchmarkname? \textbf{(2)} How does \methodname{}'s performance compare with other SOTA multi-turn RL algorithms for training LLM agents on reasoning-heavy tasks? \textbf{(3)} How does the use of asymmetric information help with credit assignments? \textbf{(4)} What are the best algorithmic choices for effectively taking advantage of LLM's reasoning and generalization capability to perform credit assignments? \textbf{(5)} How does \methodname{} scale as the number of training samples increase compared to baselines?

\begin{table*}[!t]
    \centering
    \small
    \vspace{-2mm}
    \setlength{\tabcolsep}{5.0pt}
        \caption{\textbf{Comparisons of different LLMs and multi-turn RL algorithms on \benchmarkname}. \methodname{} is able to achieve more than 6\% performance gain over other multi-turn RL algorithms, 
        enabling Llama-3.1-8B-Instruct to  be on par with larger proprietary models.}
        \vspace{0.2cm}
      \begin{adjustbox}{max width=.9\textwidth}

        \begin{tabular}{llcccc}
            \toprule
            && \multicolumn{2}{c}{\textbf{Backend Programming}} & \multicolumn{2}{c}{\textbf{Frontend Design}} \\ 
            \cmidrule(lr){3-4} \cmidrule(lr){5-6}
            && \texttt{\% Tests Passed} &\texttt{Success Rate} & \texttt{Cosine Similarity} &  \texttt{Win Rate}\\ 
            \midrule
             \multirow{4}{*}{\textsc{Single-Turn}} & Llama-3.1-8B-Instruct  & 11.8 & \phantom{0}6.9  &   63.1 & 13.6\\ 
            & Llama-3.1-70B-Instruct & 24.2 & 14.8 & 61.8 & 13.2\\ 
            & O1-Mini & 22.4 & 13.1  & 70.2 &  23.8\\ 
            & GPT-4O & 27.6  & 16.2 & 68.6  &  23.8\\
            \midrule
             \multirow{3}{*}{\textsc{SOTA LLMs}} & Llama-3.1-70B-Instruct & 48.0 & 35.0 & 73.7 & 39.8\\
             & GPT-4O & 54.6  & \bf 40.4 & 78.1  &  50.0\\
            & O1-Mini & 43.2 & 30.3  & 77.5 & 47.2  \\ 
            \midrule
            \multirow{4}{*}{Llama-3.1-8B-Instruct} & Zero-Shot &  34.2  & 22.4 & 72.4 & 33.8\\
            & Rejection Fine-Tuning &  40.9  & 28.2 & 75.2  & 38.6\\
            & Multi-Turn DPO &  48.0  & 34.4 & 76.9 & 42.8\\
            & \methodname{} (ours) &  \textbf{56.8}  & \textbf{40.4} & \textbf{77.7} & \textbf{48.2}\\
            \bottomrule
        \end{tabular}
        \end{adjustbox}
        \label{tab:main-table}
        \vspace{-0.10cm}
\end{table*}

\vspace{-0.1cm}
\subsection{Experimental Setup}
\vspace{-0.1cm}

\textbf{Baseline comparisons.} We compare \methodname{} with a variety of SOTA LLMs and multi-turn RL algorithms on \benchmarkname. We consider \textbf{Llama-3.1-8B-Instruct} and \textbf{Llama-3.1-70B-Instruct}~\citep{dubey2024llama3herdmodels} as representatives of SOTA open-source LLMs and \textbf{GPT4-O} and \textbf{O1-Mini} as representatives of SOTA proprietary LLMs. We test these models in both a \textbf{single-turn} and a \textbf{collaborative} setting to understand how LLM agents can benefit from multi-turn collaborative interactions on \benchmarkname. We compare different RL algorithms based on Llama-3.1-8B-Instruct. We first consider a simple yet effective baseline \textbf{Rejection Fine-Tuning} widely used for LLM agent fine-tuning~\citep{zhou2024proposeragentevaluatorpaeautonomousskilldiscovery, dong2023raftrewardrankedfinetuning}, where Supervised Fine-Tuning (SFT) is performed on successful trajectories to minimize the negative log-likelihood loss. Furthermore, we consider a recent effective baseline \textbf{Multi-Turn DPO} that applies a variant of DPO~\citep{rafailov2024directpreferenceoptimizationlanguage} to the multi-turn setting~\citep{xiong2024buildingmathagentsmultiturn, song2024trialerrorexplorationbasedtrajectory}. Multi-Turn DPO first constructs contrastive trajectory pairs where the chosen trajectory achieves a higher trajectory reward compared to the rejected trajectory, and uses the DPO loss to directly optimize the policy without using a critic.

\textbf{\methodname{}} first trains a turn-wise advantage model using the same model architecture with access to training-time information and then optimizes the policy with respect to rewards given by the turn-level advantage model as described in \autoref{sec:method}. We use the reference code and web page as training-time information for Backend Programming and Frontend Design respectively. Due to the multi-modal nature of the reference web page, we instantiate the advantage LLM with a similar sized VLM Qwen2-VL-7B-Instruct~\citep{yang2024qwen2technicalreport} with a regression head on top of the mean representations of all visual and text tokens.

Note that our experiments focus on the setting of learning from historically collected data (offline setting), and thus RL algorithms like PPO~\citep{schulman2017proximalpolicyoptimizationalgorithms} and REINFORCE~\citep{Williams2004SimpleSG} that require on-policy data collection do not apply. This is because in the real world, online human collaboration data requires extensive human annotations and can be costly to obtain.

\textbf{Evaluation metrics.} Each task on \textbf{Backend Programming} comes with 10 unit tests, and we report the average percentage of tests passed and task success rate where all 10 unit tests for a task are passed. We report the average cosine similarity of the final web page and the reference web page measured by the image representations from Clip-vit-base-patch32~\citep{radford2021learningtransferablevisualmodels} for \textbf{Frontend Design}. To provide a more interpretable metric, we also include the win rate with respect to GPT-4O, where the model that lands on a web page closer to the reference web page as measured by cosine similarity wins for each task.

\vspace{-0.1cm}
\subsection{Comparisons on \benchmarkname}
\vspace{-0.1cm}
\autoref{tab:main-table} shows the performance comparison of different SOTA LLMs and multi-turn RL algorithms across different tasks on \benchmarkname. First, comparing ``single-turn'' results and the other collaborative results, we note that multi-turn collaborations can greatly enhance the performance of LLM agents for artifact creations by more closely aligning the final product with the reference ``expectations'' of human simulators.
If the agent has to directly produce the final product in a single turn (top rows), even the best-performing GPT-4O can only achieve 16.2\%. In contrast, the success rates for all models are doubled (e.g., the success rate for Llama-3.1-8B-Instruct increases from 6.9\% to 22.4\%) if they are given the chance to interact with human simulators for several turns and gather more information. Nonetheless, multi-turn collaboration remains a challenging task even for proprietary LLMs like GPT-4O and O1-Mini, which can only achieve a success rate of 40.4\% and 30.3\%, respectively. Despite the improvements of O1-Mini on symbolic reasoning tasks such as math and coding, we observe that these improvements do not directly result in taking a better strategy for multi-turn collaborative agents, indicating that downstream fine-tuning is still necessary for LLMs to optimize collaboration with humans.

\vspace{-0.1cm}
\subsection{Comparing other algorithms with \methodname{}}
\vspace{-0.1cm}

After fine-tuning with downstream data, we find that even the most naive RL algorithm, Rejection Fine-Tuning, can improve the performance on both tasks, with 5.8\% and 4.8\% improvements on Backend Programming success rate and Frontend Design win rate, respectively. However, we observe that oftentimes Rejection Fine-Tuning simply teaches the LLM to ``memorize'' the solution for each training task without learning a generalizable strategy to tackle a novel test tasks. While this issue is mitigated for Multi-Turn DPO by introducing ``negative gradients'' for the rejected trajectories, the improvement is still limited without proper credit assignments over a long horizon. By explicitly training a turn-level reward model to perform credit assignments through the use of training-time information, we observe a significant gain of \methodname{} over Multi-Turn DPO on both tasks (6\% improvement in success rate for Backend Programming and 5.4\% improvement in win rate for Frontend Design). In fact, the resulting model of \methodname{} using Llama-3.1-8B-Instruct not only matches the performance of Llama-3.1-70B-Instruct with more than 8 times parameters but also achieves competitive performance with SOTA proprietary models like GPT-4O and O1-Mini.

\vspace{-0.1cm}
\subsection{Analysis} \label{sec:discussions}
\vspace{-0.1cm}
With the advantage of \methodname{} over baselines presented in \autoref{tab:main-table}, this section presents analytical results to understand this advantage and compare alternative designs.

\begin{figure*}[!h]
     \centering
    \includegraphics[width=0.96\textwidth]{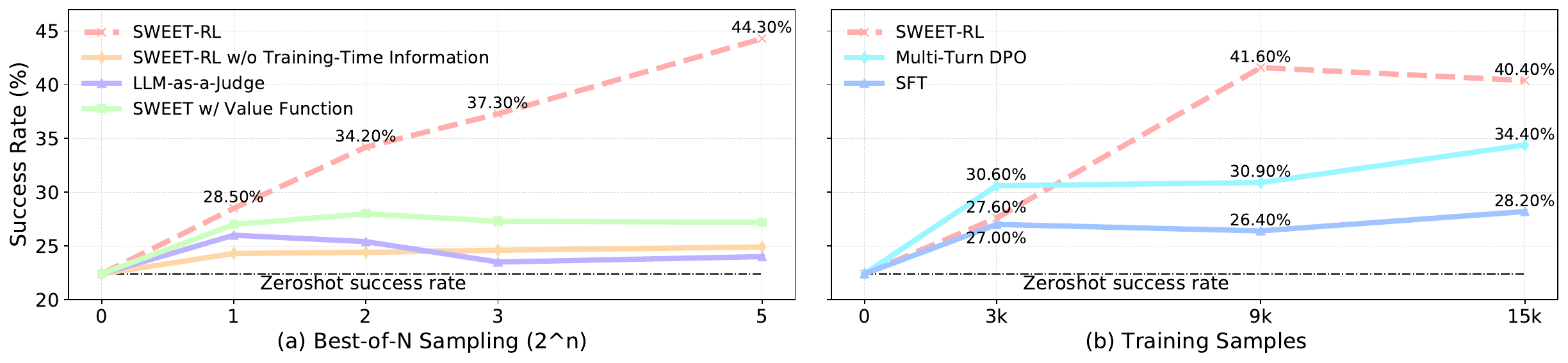}
        \caption{\emph{(a)} \textbf{Scaling curve of Best-of-N sampling with respect to different step reward model on Backend Programming.} Results show that SWEET can best tell good actions on a turn-wise basis, resulting in the best scaling with respect to Best-of-N sampling. Note that this curve is different from test-time scaling curve because SWEET exploits additional training-time information. \emph{(b)} \textbf{Performance scaling of different multi-turn RL algorithms with respect to the amount of fine-tuning data on Backend Programming.} While \methodname{} takes more data initially to learn a reliable critic, it quickly catches up and achieves a better converging performance.}
        \label{fig:scaling} 
        \vspace{-0.3cm}
\end{figure*}

\textbf{How should we use training-time information to help credit assignment?} We carry out Best-of-N sampling experiments on Backend Programming to directly compare the capability of different methods to perform credit assignments. At each turn, N candidate actions are sampled from a fixed actor Llama-3.1-8B-Instruct and different methods are used to choose the best action to be executed. We compare our reward model in \methodname{} with three other natural choices: 1) \methodname{} w/o Training-Time Information -- the same as \methodname{} except that it only has the interaction history as inputs without access to the reference solution, 2) LLM-as-a-Judge uses a Llama-3.1-8B-Instruct to  compare pairwise the quality of each action based on the interaction history and reference solution, and 3) \methodname{} w/ Value Function where a regular classification head on top of s Llama-3.1-8B-Instruct backbone is trained to predict the expected success rate given the interaction history and the reference solution with a binary classification loss. Performance comparisons are presented in \autoref{fig:scaling}\emph{(a)} and qualitative comparisons are presented in \autoref{app:credit_assignment_qualitative}. First of all, we observe that the use of training-time information can significantly enhance the capability to perform credit assignment, evidenced by the huge performance gap between \methodname{} and \methodname{} w/o Training-Time Information. While Best-of-N sampling with respect to a fixed LLM-as-a-Judge can result in some improvements over the zeroshot success rate, this improvement is limited. Qualitatively, we found that a fixed LLM-as-a-Judge can easily get distracted by the length and format of the response without actually attending to its utility for task success. Finally, while being standard practice in the deep RL literature, the use of a value function fails to achieve comparable scaling performance compared to \methodname{}. This shows the importance of the careful RL algorithmic choices of \methodname{} and that the go-to practice of training a value function may generalize poorly in unseen tasks.

\textbf{What are the best parameterization choices for the critic to perform credit assignment?} In \autoref{tab:ablation}, we perform ablation experiments to understand the effect of different parameterizations for the advantage function of the critic. In particular, we consider two alternative parameterizations: 1) \methodname{} w/ Regression Head uses a regression head on top of the mean pooled representations of the last hidden state across all tokens, 
2) \methodname{} w/o Training-Time Information only uses the interaction history as inputs without access to the reference solution, 
and 3) \methodname{} w/o Normalization where we do not perform the normalization step of division by the number of tokens in \autoref{eqn:advantage_parameterization}. As shown in \autoref{tab:ablation}, similar to the conclusion from the previous section, the use of a regression head does not generalize well compared to \methodname{} and training-time information significantly improves the performance. Without the additional normalization step, we find that the trained actor quickly collapses by generating shorter and shorter responses, showing the importance of carefully chosen parameterizations for multi-turn RL algorithms.

\begin{table}[ht]
    \centering
    \small
    \setlength{\tabcolsep}{2.0pt}
        \caption{\textbf{Ablation study on different parameterizations of the critic.} Results show that the parameterization of using the mean log probability significantly outperform the other natural choices.}
        \vspace{4mm}
        \label{tab:ablation}
  \begin{adjustbox}{max width=\linewidth}
        \begin{tabular}{lcc}
            \toprule
            & {\% Tests Passed} &{Success Rate}\\ 
            \midrule
            Rejection Fine-Tuning &  40.9  & 28.2 \\
            Multi-Turn DPO &  48.0  & 34.4 \\
             \methodname{} &  \textbf{56.8}  & \textbf{40.4} \\
             \midrule
             \methodname{} w/ Regression Head &  45.3  & 36.2 \\
             \methodname{} w/o Train-Time Info. &  44.0  & 31.2 \\
             \methodname{} w/o Normalization &  \phantom{0}4.2  & \phantom{0}3.6 \\
            \bottomrule
        \end{tabular}
    \end{adjustbox}
\end{table}

\textbf{How does \methodname{} scale with the amount of fine-tuning data?} Additionally, we compare the scaling performance of \methodname{} compared to the multi-turn RL baselines Rejection Fine-Tuning and Multi-Turn DPO. Results are presented in ~\autoref{fig:scaling}\emph{(b)}. Although \methodname{} requires more data to train a reliable critic for performing credit assignment where it under-performs Multi-Turn DPO with 3k fine-tuning samples available, \methodname{} quickly catches up with more samples once the turn-wise critic is trained and results in a significantly improved converging performance.

\textbf{How does \methodname{} work for different model architectures and off-policy data?} Finally, \autoref{tab:70b} presents additional comparison experiments on Backend Programming using a stronger base model Llama-3.1-70B-Instruct to study how \methodname{} works across different model architectures. We use the same offline data generated by Llama-3.1-8B as \autoref{tab:main-table} to understand the effectiveness of different methods for taking advantage of off-policy data (i.e. offline generated by a different model). We first observe that Rejection Fine-Tuning fails to learn from data generated by a worse model, with the success rate dropping from 35.0\% to 31.9\%. This is potentially because the objective of Rejection Fine-Tuning forces Llama-3.1-70B-Instruct to imitate sub-optimal trajectories generated from the worse Llama-3.1-8B-Instruct word-by-word. While Multi-turn DPO is able to achieve a big improvement even using off-policy data generated from an inferior model (35.0\% to 41.8\% in success rate), \methodname{} still maintains a similar gap of 3.8\% through performing explicit credit assignments with training-time information (41.8\% compared to 45.6\% in success rate).

\begin{table}[ht]
    \centering
    \small
    \setlength{\tabcolsep}{2.0pt}
        \caption{\textbf{Comparison results on Backend Programming using Llama-3.1-70B-Instruct as the base model. } Results show that \methodname{} achieves a similar gain over the baselines when using stronger Llama-3.1-70B-Instruct as the base model.}
        \vspace{4mm}
        \label{tab:70b}
  \begin{adjustbox}{max width=\linewidth}
        \begin{tabular}{lcc}
            \toprule
            Llama-3.1-70B-Instruct & {\% Tests Passed} &{Success Rate}\\ 
            \midrule
            Zeroshot &  48.0  & 35.0 \\
            Rejection Fine-Tuning &  45.5  & 31.9 \\
            Multi-Turn DPO &  56.7  &  41.8\\
             \midrule
             \methodname{} & \textbf{60.2}  & \textbf{45.6} \\
            \bottomrule
        \end{tabular}
    \end{adjustbox}
\end{table}

\vspace{-0.1cm}
\section{Conclusion}
\vspace{-0.1cm}

To advance the development of effective multi-turn RL algorithms that perform effective credit assignments, this paper first introduces a benchmark, \benchmarkname{}, focusing on the realistic domain of collaborative artifact creation. \benchmarkname{} is the first LLM agent benchmark designed to validate multi-turn RL algorithms for reasoning-intensive tasks with minimum engineering overhead. Building upon \benchmarkname{}, we develop a novel multi-turn RL algorithm, \methodname{}, leveraging additional training-time information and appropriate algorithmic choices, achieving significantly improved performance over SOTA baselines in this domain. Our experiment results on \benchmarkname{} show that off-the-shelf deep RL methods for multi-turn LLM agents can lead to poor generalization performance. While \methodname{} serves as a preliminary step for closing this gap, there are a lot of future research opportunities to develop a better multi-turn RL algorithm in this important area of LLM agents.

\section*{Impact Statement}
This paper advances the development of more effective multi-turn RL algorithms and better human-agent collaborations. An effective collaborative LLM may significantly improve human productivity in many areas such as content creation. However, various safety concerns may arise as LLM agents take over more tasks from humans where they might be subject to malicious attacks or conduct unexpected behaviors. We leave this important direction for future research.


\bibliography{paper}

\begin{thebibliography}{66}
\providecommand{\natexlab}[1]{#1}
\providecommand{\url}[1]{\texttt{#1}}
\expandafter\ifx\csname urlstyle\endcsname\relax
  \providecommand{\doi}[1]{doi: #1}\else
  \providecommand{\doi}{doi: \begingroup \urlstyle{rm}\Url}\fi

\bibitem[Abdulhai et~al.(2023)Abdulhai, White, Snell, Sun, Hong, Zhai, Xu, and Levine]{abdulhai2023lmrlgymbenchmarksmultiturn}
Marwa Abdulhai, Isadora White, Charlie Snell, Charles Sun, Joey Hong, Yuexiang Zhai, Kelvin Xu, and Sergey Levine.
\newblock Lmrl gym: Benchmarks for multi-turn reinforcement learning with language models, 2023.
\newblock \url{https://arxiv.org/abs/2311.18232}.

\bibitem[Agarwal et~al.(2019)Agarwal, Jiang, and Kakade]{agarwal2019reinforcement}
Alekh Agarwal, Nan Jiang, and Sham~M Kakade.
\newblock Reinforcement learning: Theory and algorithms.
\newblock 2019.

\bibitem[Bai et~al.(2024)Bai, Zhou, Cemri, Pan, Suhr, Levine, and Kumar]{bai2024digirltraininginthewilddevicecontrol}
Hao Bai, Yifei Zhou, Mert Cemri, Jiayi Pan, Alane Suhr, Sergey Levine, and Aviral Kumar.
\newblock Digirl: Training in-the-wild device-control agents with autonomous reinforcement learning, 2024.
\newblock \url{https://arxiv.org/abs/2406.11896}.

\bibitem[Bradley and Terry(1952)]{Bradley1952RankAO}
Ralph~Allan Bradley and Milton~E. Terry.
\newblock Rank analysis of incomplete block designs: I. the method of paired comparisons.
\newblock \emph{Biometrika}, 39:\penalty0 324, 1952.
\newblock \url{https://api.semanticscholar.org/CorpusID:125209808}.

\bibitem[Casper et~al.(2023)Casper, Davies, Shi, Gilbert, Scheurer, Rando, Freedman, Korbak, Lindner, Freire, Wang, Marks, Segerie, Carroll, Peng, Christoffersen, Damani, Slocum, Anwar, Siththaranjan, Nadeau, Michaud, Pfau, Krasheninnikov, Chen, Langosco, Hase, Bıyık, Dragan, Krueger, Sadigh, and Hadfield-Menell]{casper2023openproblemsfundamentallimitations}
Stephen Casper, Xander Davies, Claudia Shi, Thomas~Krendl Gilbert, Jérémy Scheurer, Javier Rando, Rachel Freedman, Tomasz Korbak, David Lindner, Pedro Freire, Tony Wang, Samuel Marks, Charbel-Raphaël Segerie, Micah Carroll, Andi Peng, Phillip Christoffersen, Mehul Damani, Stewart Slocum, Usman Anwar, Anand Siththaranjan, Max Nadeau, Eric~J. Michaud, Jacob Pfau, Dmitrii Krasheninnikov, Xin Chen, Lauro Langosco, Peter Hase, Erdem Bıyık, Anca Dragan, David Krueger, Dorsa Sadigh, and Dylan Hadfield-Menell.
\newblock Open problems and fundamental limitations of reinforcement learning from human feedback, 2023.
\newblock \url{https://arxiv.org/abs/2307.15217}.

\bibitem[Cheng et~al.(2023)Cheng, Kolobov, Misra, Nie, and Swaminathan]{cheng2023llfbenchbenchmarkinteractivelearning}
Ching-An Cheng, Andrey Kolobov, Dipendra Misra, Allen Nie, and Adith Swaminathan.
\newblock Llf-bench: Benchmark for interactive learning from language feedback, 2023.
\newblock \url{https://arxiv.org/abs/2312.06853}.

\bibitem[Christiano et~al.(2023)Christiano, Leike, Brown, Martic, Legg, and Amodei]{christiano2023deepreinforcementlearninghuman}
Paul Christiano, Jan Leike, Tom~B. Brown, Miljan Martic, Shane Legg, and Dario Amodei.
\newblock Deep reinforcement learning from human preferences, 2023.
\newblock \url{https://arxiv.org/abs/1706.03741}.

\bibitem[Deng et~al.(2023)Deng, Gu, Zheng, Chen, Stevens, Wang, Sun, and Su]{deng2023mind2webgeneralistagentweb}
Xiang Deng, Yu~Gu, Boyuan Zheng, Shijie Chen, Samuel Stevens, Boshi Wang, Huan Sun, and Yu~Su.
\newblock Mind2web: Towards a generalist agent for the web, 2023.
\newblock \url{https://arxiv.org/abs/2306.06070}.

\bibitem[Dong et~al.(2023)Dong, Xiong, Goyal, Zhang, Chow, Pan, Diao, Zhang, Shum, and Zhang]{dong2023raftrewardrankedfinetuning}
Hanze Dong, Wei Xiong, Deepanshu Goyal, Yihan Zhang, Winnie Chow, Rui Pan, Shizhe Diao, Jipeng Zhang, Kashun Shum, and Tong Zhang.
\newblock Raft: Reward ranked finetuning for generative foundation model alignment, 2023.
\newblock \url{https://arxiv.org/abs/2304.06767}.

\bibitem[GeminiTeam(2024)]{geminiteam2024geminifamilyhighlycapable}
GeminiTeam.
\newblock Gemini: A family of highly capable multimodal models, 2024.
\newblock \url{https://arxiv.org/abs/2312.11805}.

\bibitem[Gur et~al.(2021)Gur, Furuta, Huang, Safdari, Matsuo, Eck, and Faust]{gur2023webagent}
Izzeddin Gur, Hiroki Furuta, Austin~V. Huang, Mustafa Safdari, Yutaka Matsuo, Douglas Eck, and Aleksandra Faust.
\newblock A real-world webagent with planning, long context understanding, and program synthesis, 2021.

\bibitem[Hwang et~al.(2024)Hwang, Kim, Kim, Ye, and Seo]{hwang2024selfexploreenhancingmathematicalreasoning}
Hyeonbin Hwang, Doyoung Kim, Seungone Kim, Seonghyeon Ye, and Minjoon Seo.
\newblock Self-explore: Enhancing mathematical reasoning in language models with fine-grained rewards, 2024.
\newblock \url{https://arxiv.org/abs/2404.10346}.

\bibitem[Jiang et~al.(2024)Jiang, JU, Cohen, Mitts, Foss, Kao, Li, and Tian]{jiang2024delegationdesigningidealagentic}
Song Jiang, Da~JU, Andrew Cohen, Sasha Mitts, Aaron Foss, Justine~T Kao, Xian Li, and Yuandong Tian.
\newblock Towards full delegation: Designing ideal agentic behaviors for travel planning, 2024.
\newblock \url{https://arxiv.org/abs/2411.13904}.

\bibitem[Jimenez et~al.(2024)Jimenez, Yang, Wettig, Yao, Pei, Press, and Narasimhan]{jimenez2024swebenchlanguagemodelsresolve}
Carlos~E. Jimenez, John Yang, Alexander Wettig, Shunyu Yao, Kexin Pei, Ofir Press, and Karthik Narasimhan.
\newblock Swe-bench: Can language models resolve real-world github issues?, 2024.
\newblock \url{https://arxiv.org/abs/2310.06770}.

\bibitem[Koh et~al.(2024)Koh, Lo, Jang, Duvvur, Lim, Huang, Neubig, Zhou, Salakhutdinov, and Fried]{koh2024visualwebarenaevaluatingmultimodalagents}
Jing~Yu Koh, Robert Lo, Lawrence Jang, Vikram Duvvur, Ming~Chong Lim, Po-Yu Huang, Graham Neubig, Shuyan Zhou, Ruslan Salakhutdinov, and Daniel Fried.
\newblock Visualwebarena: Evaluating multimodal agents on realistic visual web tasks, 2024.
\newblock \url{https://arxiv.org/abs/2401.13649}.

\bibitem[Kumar et~al.(2024)Kumar, Zhuang, Agarwal, Su, Co-Reyes, Singh, Baumli, Iqbal, Bishop, Roelofs, Zhang, McKinney, Shrivastava, Paduraru, Tucker, Precup, Behbahani, and Faust]{kumar2024traininglanguagemodelsselfcorrect}
Aviral Kumar, Vincent Zhuang, Rishabh Agarwal, Yi~Su, John~D Co-Reyes, Avi Singh, Kate Baumli, Shariq Iqbal, Colton Bishop, Rebecca Roelofs, Lei~M Zhang, Kay McKinney, Disha Shrivastava, Cosmin Paduraru, George Tucker, Doina Precup, Feryal Behbahani, and Aleksandra Faust.
\newblock Training language models to self-correct via reinforcement learning, 2024.
\newblock \url{https://arxiv.org/abs/2409.12917}.

\bibitem[Laurençon et~al.(2024)Laurençon, Tronchon, and Sanh]{laurençon2024unlockingconversionwebscreenshots}
Hugo Laurençon, Léo Tronchon, and Victor Sanh.
\newblock Unlocking the conversion of web screenshots into html code with the websight dataset, 2024.
\newblock \url{https://arxiv.org/abs/2403.09029}.

\bibitem[Li et~al.(2024)Li, Fang, Smyrnis, Ivgi, Jordan, Gadre, Bansal, Guha, Keh, Arora, Garg, Xin, Muennighoff, Heckel, Mercat, Chen, Gururangan, Wortsman, Albalak, Bitton, Nezhurina, Abbas, Hsieh, Ghosh, Gardner, Kilian, Zhang, Shao, Pratt, Sanyal, Ilharco, Daras, Marathe, Gokaslan, Zhang, Chandu, Nguyen, Vasiljevic, Kakade, Song, Sanghavi, Faghri, Oh, Zettlemoyer, Lo, El-Nouby, Pouransari, Toshev, Wang, Groeneveld, Soldaini, Koh, Jitsev, Kollar, Dimakis, Carmon, Dave, Schmidt, and Shankar]{li2024datacomplmsearchgenerationtraining}
Jeffrey Li, Alex Fang, Georgios Smyrnis, Maor Ivgi, Matt Jordan, Samir Gadre, Hritik Bansal, Etash Guha, Sedrick Keh, Kushal Arora, Saurabh Garg, Rui Xin, Niklas Muennighoff, Reinhard Heckel, Jean Mercat, Mayee Chen, Suchin Gururangan, Mitchell Wortsman, Alon Albalak, Yonatan Bitton, Marianna Nezhurina, Amro Abbas, Cheng-Yu Hsieh, Dhruba Ghosh, Josh Gardner, Maciej Kilian, Hanlin Zhang, Rulin Shao, Sarah Pratt, Sunny Sanyal, Gabriel Ilharco, Giannis Daras, Kalyani Marathe, Aaron Gokaslan, Jieyu Zhang, Khyathi Chandu, Thao Nguyen, Igor Vasiljevic, Sham Kakade, Shuran Song, Sujay Sanghavi, Fartash Faghri, Sewoong Oh, Luke Zettlemoyer, Kyle Lo, Alaaeldin El-Nouby, Hadi Pouransari, Alexander Toshev, Stephanie Wang, Dirk Groeneveld, Luca Soldaini, Pang~Wei Koh, Jenia Jitsev, Thomas Kollar, Alexandros~G. Dimakis, Yair Carmon, Achal Dave, Ludwig Schmidt, and Vaishaal Shankar.
\newblock Datacomp-lm: In search of the next generation of training sets for language models, 2024.
\newblock \url{https://arxiv.org/abs/2406.11794}.

\bibitem[Lightman et~al.(2023)Lightman, Kosaraju, Burda, Edwards, Baker, Lee, Leike, Schulman, Sutskever, and Cobbe]{lightman2023letsverifystepstep}
Hunter Lightman, Vineet Kosaraju, Yura Burda, Harri Edwards, Bowen Baker, Teddy Lee, Jan Leike, John Schulman, Ilya Sutskever, and Karl Cobbe.
\newblock Let's verify step by step, 2023.
\newblock \url{https://arxiv.org/abs/2305.20050}.

\bibitem[Lin et~al.(2024)Lin, Tomlin, Andreas, and Eisner]{lin2024decisionorienteddialoguehumanaicollaboration}
Jessy Lin, Nicholas Tomlin, Jacob Andreas, and Jason Eisner.
\newblock Decision-oriented dialogue for human-ai collaboration, 2024.
\newblock \url{https://arxiv.org/abs/2305.20076}.

\bibitem[Lin et~al.(2025)Lin, Jin, Xu, Wu, Sukhbaatar, Zhu, He, Chen, Weston, Tian, et~al.]{lin2025step}
Yen-Ting Lin, Di~Jin, Tengyu Xu, Tianhao Wu, Sainbayar Sukhbaatar, Chen Zhu, Yun He, Yun-Nung Chen, Jason Weston, Yuandong Tian, et~al.
\newblock Step-{KTO}: Optimizing mathematical reasoning through stepwise binary feedback.
\newblock \emph{arXiv preprint arXiv:2501.10799}, 2025.

\bibitem[Liu et~al.(2024)Liu, Ji, Zheng, Wu, Dun, Gu, and Yan]{liu2024enhancingmultistepreasoningabilities}
Guanlin Liu, Kaixuan Ji, Renjie Zheng, Zheng Wu, Chen Dun, Quanquan Gu, and Lin Yan.
\newblock Enhancing multi-step reasoning abilities of language models through direct q-function optimization, 2024.
\newblock \url{https://arxiv.org/abs/2410.09302}.

\bibitem[Liu et~al.(2023)Liu, Yu, Zhang, Xu, Lei, Lai, Gu, Ding, Men, Yang, Zhang, Deng, Zeng, Du, Zhang, Shen, Zhang, Su, Sun, Huang, Dong, and Tang]{liu2023agentbenchevaluatingllmsagents}
Xiao Liu, Hao Yu, Hanchen Zhang, Yifan Xu, Xuanyu Lei, Hanyu Lai, Yu~Gu, Hangliang Ding, Kaiwen Men, Kejuan Yang, Shudan Zhang, Xiang Deng, Aohan Zeng, Zhengxiao Du, Chenhui Zhang, Sheng Shen, Tianjun Zhang, Yu~Su, Huan Sun, Minlie Huang, Yuxiao Dong, and Jie Tang.
\newblock Agentbench: Evaluating llms as agents, 2023.
\newblock \url{https://arxiv.org/abs/2308.03688}.

\bibitem[Llama3Team(2024)]{dubey2024llama3herdmodels}
Llama3Team.
\newblock The llama 3 herd of models, 2024.
\newblock \url{https://arxiv.org/abs/2407.21783}.

\bibitem[Luo et~al.(2024)Luo, Liu, Liu, Phatale, Guo, Lara, Li, Shu, Zhu, Meng, Sun, and Rastogi]{luo2024improvemathematicalreasoninglanguage}
Liangchen Luo, Yinxiao Liu, Rosanne Liu, Samrat Phatale, Meiqi Guo, Harsh Lara, Yunxuan Li, Lei Shu, Yun Zhu, Lei Meng, Jiao Sun, and Abhinav Rastogi.
\newblock Improve mathematical reasoning in language models by automated process supervision, 2024.
\newblock \url{https://arxiv.org/abs/2406.06592}.

\bibitem[Mnih et~al.(2013)Mnih, Kavukcuoglu, Silver, Graves, Antonoglou, Wierstra, and Riedmiller]{mnih2013playingatarideepreinforcement}
Volodymyr Mnih, Koray Kavukcuoglu, David Silver, Alex Graves, Ioannis Antonoglou, Daan Wierstra, and Martin Riedmiller.
\newblock Playing atari with deep reinforcement learning, 2013.
\newblock \url{https://arxiv.org/abs/1312.5602}.

\bibitem[Nachum et~al.(2017)Nachum, Norouzi, Xu, and Schuurmans]{nachum2017bridginggapvaluepolicy}
Ofir Nachum, Mohammad Norouzi, Kelvin Xu, and Dale Schuurmans.
\newblock Bridging the gap between value and policy based reinforcement learning, 2017.
\newblock \url{https://arxiv.org/abs/1702.08892}.

\bibitem[OpenAI(2024)]{openai2024gpt4technicalreport}
OpenAI.
\newblock Gpt-4 technical report, 2024.
\newblock \url{https://arxiv.org/abs/2303.08774}.

\bibitem[Ouyang et~al.(2022)Ouyang, Wu, Jiang, Almeida, Wainwright, Mishkin, Zhang, Agarwal, Slama, Ray, Schulman, Hilton, Kelton, Miller, Simens, Askell, Welinder, Christiano, Leike, and Lowe]{ouyang2022traininglanguagemodelsfollow}
Long Ouyang, Jeff Wu, Xu~Jiang, Diogo Almeida, Carroll~L. Wainwright, Pamela Mishkin, Chong Zhang, Sandhini Agarwal, Katarina Slama, Alex Ray, John Schulman, Jacob Hilton, Fraser Kelton, Luke Miller, Maddie Simens, Amanda Askell, Peter Welinder, Paul Christiano, Jan Leike, and Ryan Lowe.
\newblock Training language models to follow instructions with human feedback, 2022.
\newblock \url{https://arxiv.org/abs/2203.02155}.

\bibitem[Pan et~al.(2024)Pan, Wang, Neubig, Jaitly, Ji, Suhr, and Zhang]{pan2024trainingsoftwareengineeringagents}
Jiayi Pan, Xingyao Wang, Graham Neubig, Navdeep Jaitly, Heng Ji, Alane Suhr, and Yizhe Zhang.
\newblock Training software engineering agents and verifiers with swe-gym, 2024.
\newblock \url{https://arxiv.org/abs/2412.21139}.

\bibitem[Pang et~al.(2024)Pang, Yuan, Cho, He, Sukhbaatar, and Weston]{pang2024iterativereasoningpreferenceoptimization}
Richard~Yuanzhe Pang, Weizhe Yuan, Kyunghyun Cho, He~He, Sainbayar Sukhbaatar, and Jason Weston.
\newblock Iterative reasoning preference optimization, 2024.
\newblock \url{https://arxiv.org/abs/2404.19733}.

\bibitem[Pinto et~al.(2017)Pinto, Andrychowicz, Welinder, Zaremba, and Abbeel]{pinto2017asymmetricactorcriticimagebased}
Lerrel Pinto, Marcin Andrychowicz, Peter Welinder, Wojciech Zaremba, and Pieter Abbeel.
\newblock Asymmetric actor critic for image-based robot learning, 2017.
\newblock \url{https://arxiv.org/abs/1710.06542}.

\bibitem[Radford et~al.(2021)Radford, Kim, Hallacy, Ramesh, Goh, Agarwal, Sastry, Askell, Mishkin, Clark, Krueger, and Sutskever]{radford2021learningtransferablevisualmodels}
Alec Radford, Jong~Wook Kim, Chris Hallacy, Aditya Ramesh, Gabriel Goh, Sandhini Agarwal, Girish Sastry, Amanda Askell, Pamela Mishkin, Jack Clark, Gretchen Krueger, and Ilya Sutskever.
\newblock Learning transferable visual models from natural language supervision, 2021.
\newblock \url{https://arxiv.org/abs/2103.00020}.

\bibitem[Rafailov et~al.(2024{\natexlab{a}})Rafailov, Hejna, Park, and Finn]{rafailov2024rqlanguagemodel}
Rafael Rafailov, Joey Hejna, Ryan Park, and Chelsea Finn.
\newblock From $r$ to $q^*$: Your language model is secretly a q-function, 2024{\natexlab{a}}.
\newblock \url{https://arxiv.org/abs/2404.12358}.

\bibitem[Rafailov et~al.(2024{\natexlab{b}})Rafailov, Sharma, Mitchell, Ermon, Manning, and Finn]{rafailov2024directpreferenceoptimizationlanguage}
Rafael Rafailov, Archit Sharma, Eric Mitchell, Stefano Ermon, Christopher~D. Manning, and Chelsea Finn.
\newblock Direct preference optimization: Your language model is secretly a reward model, 2024{\natexlab{b}}.
\newblock \url{https://arxiv.org/abs/2305.18290}.

\bibitem[Rawles et~al.(2023)Rawles, Li, Rodriguez, Riva, and Lillicrap]{rawles2023androidwildlargescaledataset}
Christopher Rawles, Alice Li, Daniel Rodriguez, Oriana Riva, and Timothy Lillicrap.
\newblock Android in the wild: A large-scale dataset for android device control, 2023.
\newblock \url{https://arxiv.org/abs/2307.10088}.

\bibitem[Rawles et~al.(2024)Rawles, Clinckemaillie, Chang, Waltz, Lau, Fair, Li, Bishop, Li, Campbell-Ajala, Toyama, Berry, Tyamagundlu, Lillicrap, and Riva]{rawles2024androidworlddynamicbenchmarkingenvironment}
Christopher Rawles, Sarah Clinckemaillie, Yifan Chang, Jonathan Waltz, Gabrielle Lau, Marybeth Fair, Alice Li, William Bishop, Wei Li, Folawiyo Campbell-Ajala, Daniel Toyama, Robert Berry, Divya Tyamagundlu, Timothy Lillicrap, and Oriana Riva.
\newblock Androidworld: A dynamic benchmarking environment for autonomous agents, 2024.
\newblock \url{https://arxiv.org/abs/2405.14573}.

\bibitem[Salter et~al.(2019)Salter, Rao, Wulfmeier, Hadsell, and Posner]{Salter2019AttentionPR}
Sasha Salter, Dushyant Rao, Markus Wulfmeier, Raia Hadsell, and Ingmar Posner.
\newblock Attention privileged reinforcement learning for domain transfer.
\newblock \emph{ArXiv}, abs/1911.08363, 2019.
\newblock \url{https://api.semanticscholar.org/CorpusID:215835372}.

\bibitem[Schulman et~al.(2017)Schulman, Wolski, Dhariwal, Radford, and Klimov]{schulman2017proximalpolicyoptimizationalgorithms}
John Schulman, Filip Wolski, Prafulla Dhariwal, Alec Radford, and Oleg Klimov.
\newblock Proximal policy optimization algorithms, 2017.
\newblock \url{https://arxiv.org/abs/1707.06347}.

\bibitem[Setlur et~al.(2024)Setlur, Nagpal, Fisch, Geng, Eisenstein, Agarwal, Agarwal, Berant, and Kumar]{setlur2024rewardingprogressscalingautomated}
Amrith Setlur, Chirag Nagpal, Adam Fisch, Xinyang Geng, Jacob Eisenstein, Rishabh Agarwal, Alekh Agarwal, Jonathan Berant, and Aviral Kumar.
\newblock Rewarding progress: Scaling automated process verifiers for llm reasoning, 2024.
\newblock \url{https://arxiv.org/abs/2410.08146}.

\bibitem[Shao et~al.(2024)Shao, Wang, Zhu, Xu, Song, Bi, Zhang, Zhang, Li, Wu, and Guo]{shao2024deepseekmathpushinglimitsmathematical}
Zhihong Shao, Peiyi Wang, Qihao Zhu, Runxin Xu, Junxiao Song, Xiao Bi, Haowei Zhang, Mingchuan Zhang, Y.~K. Li, Y.~Wu, and Daya Guo.
\newblock Deepseekmath: Pushing the limits of mathematical reasoning in open language models, 2024.
\newblock \url{https://arxiv.org/abs/2402.03300}.

\bibitem[Si et~al.(2024)Si, Zhang, Yang, Liu, and Yang]{si2024design2codefarautomatingfrontend}
Chenglei Si, Yanzhe Zhang, Zhengyuan Yang, Ruibo Liu, and Diyi Yang.
\newblock Design2code: How far are we from automating front-end engineering?, 2024.
\newblock \url{https://arxiv.org/abs/2403.03163}.

\bibitem[Snell et~al.(2023)Snell, Kostrikov, Su, Yang, and Levine]{snell2023offlinerlnaturallanguage}
Charlie Snell, Ilya Kostrikov, Yi~Su, Mengjiao Yang, and Sergey Levine.
\newblock Offline rl for natural language generation with implicit language q learning, 2023.
\newblock \url{https://arxiv.org/abs/2206.11871}.

\bibitem[Snell et~al.(2024)Snell, Lee, Xu, and Kumar]{snell2024scalingllmtesttimecompute}
Charlie Snell, Jaehoon Lee, Kelvin Xu, and Aviral Kumar.
\newblock Scaling llm test-time compute optimally can be more effective than scaling model parameters, 2024.
\newblock \url{https://arxiv.org/abs/2408.03314}.

\bibitem[Song et~al.(2024)Song, Yin, Yue, Huang, Li, and Lin]{song2024trialerrorexplorationbasedtrajectory}
Yifan Song, Da~Yin, Xiang Yue, Jie Huang, Sujian Li, and Bill~Yuchen Lin.
\newblock Trial and error: Exploration-based trajectory optimization for llm agents, 2024.
\newblock \url{https://arxiv.org/abs/2403.02502}.

\bibitem[Sutton et~al.(1999)Sutton, McAllester, Singh, and Mansour]{REINFORCE}
Richard~S Sutton, David McAllester, Satinder Singh, and Yishay Mansour.
\newblock Policy gradient methods for reinforcement learning with function approximation.
\newblock In S.~Solla, T.~Leen, and K.~M\"{u}ller, editors, \emph{Advances in Neural Information Processing Systems}, volume~12. MIT Press, 1999.
\newblock \url{https://proceedings.neurips.cc/paper_files/paper/1999/file/464d828b85b0bed98e80ade0a5c43b0f-Paper.pdf}.

\bibitem[Szot et~al.(2024)Szot, Schwarzer, Agrawal, Mazoure, Talbott, Metcalf, Mackraz, Hjelm, and Toshev]{szot2024largelanguagemodelsgeneralizable}
Andrew Szot, Max Schwarzer, Harsh Agrawal, Bogdan Mazoure, Walter Talbott, Katherine Metcalf, Natalie Mackraz, Devon Hjelm, and Alexander Toshev.
\newblock Large language models as generalizable policies for embodied tasks, 2024.
\newblock \url{https://arxiv.org/abs/2310.17722}.

\bibitem[Uesato et~al.(2022)Uesato, Kushman, Kumar, Song, Siegel, Wang, Creswell, Irving, and Higgins]{uesato2022solvingmathwordproblems}
Jonathan Uesato, Nate Kushman, Ramana Kumar, Francis Song, Noah Siegel, Lisa Wang, Antonia Creswell, Geoffrey Irving, and Irina Higgins.
\newblock Solving math word problems with process- and outcome-based feedback, 2022.
\newblock \url{https://arxiv.org/abs/2211.14275}.

\bibitem[Wang et~al.(2024{\natexlab{a}})Wang, Hao, Dong, Zhang, Bao, Yang, and Wu]{wang2024offlinereinforcementlearningllm}
Huaijie Wang, Shibo Hao, Hanze Dong, Shenao Zhang, Yilin Bao, Ziran Yang, and Yi~Wu.
\newblock Offline reinforcement learning for llm multi-step reasoning, 2024{\natexlab{a}}.
\newblock \url{https://arxiv.org/abs/2412.16145}.

\bibitem[Wang et~al.(2024{\natexlab{b}})Wang, Wang, Liu, Chen, Yuan, Peng, and Ji]{wang2024mintevaluatingllmsmultiturn}
Xingyao Wang, Zihan Wang, Jiateng Liu, Yangyi Chen, Lifan Yuan, Hao Peng, and Heng Ji.
\newblock Mint: Evaluating llms in multi-turn interaction with tools and language feedback, 2024{\natexlab{b}}.
\newblock \url{https://arxiv.org/abs/2309.10691}.

\bibitem[Williams(2004)]{Williams2004SimpleSG}
Ronald~J. Williams.
\newblock Simple statistical gradient-following algorithms for connectionist reinforcement learning.
\newblock \emph{Machine Learning}, 8:\penalty0 229--256, 2004.
\newblock \url{https://api.semanticscholar.org/CorpusID:2332513}.

\bibitem[Wilson and Hermans(2020)]{wilson2020learningmanipulateobjectcollections}
Matthew Wilson and Tucker Hermans.
\newblock Learning to manipulate object collections using grounded state representations, 2020.
\newblock \url{https://arxiv.org/abs/1909.07876}.

\bibitem[Wu and Hu(2018)]{wu2018learningextractcoherentsummary}
Yuxiang Wu and Baotian Hu.
\newblock Learning to extract coherent summary via deep reinforcement learning, 2018.
\newblock \url{https://arxiv.org/abs/1804.07036}.

\bibitem[Xie et~al.(2024{\natexlab{a}})Xie, Zhang, Chen, Zhu, Lou, Tian, Xiao, and Su]{xie2024travelplannerbenchmarkrealworldplanning}
Jian Xie, Kai Zhang, Jiangjie Chen, Tinghui Zhu, Renze Lou, Yuandong Tian, Yanghua Xiao, and Yu~Su.
\newblock Travelplanner: A benchmark for real-world planning with language agents, 2024{\natexlab{a}}.
\newblock \url{https://arxiv.org/abs/2402.01622}.

\bibitem[Xie et~al.(2024{\natexlab{b}})Xie, Zhang, Chen, Li, Zhao, Cao, Hua, Cheng, Shin, Lei, Liu, Xu, Zhou, Savarese, Xiong, Zhong, and Yu]{xie2024osworldbenchmarkingmultimodalagents}
Tianbao Xie, Danyang Zhang, Jixuan Chen, Xiaochuan Li, Siheng Zhao, Ruisheng Cao, Toh~Jing Hua, Zhoujun Cheng, Dongchan Shin, Fangyu Lei, Yitao Liu, Yiheng Xu, Shuyan Zhou, Silvio Savarese, Caiming Xiong, Victor Zhong, and Tao Yu.
\newblock Osworld: Benchmarking multimodal agents for open-ended tasks in real computer environments, 2024{\natexlab{b}}.
\newblock \url{https://arxiv.org/abs/2404.07972}.

\bibitem[Xiong et~al.(2024)Xiong, Shi, Shen, Rosenberg, Qin, Calandriello, Khalman, Joshi, Piot, Saleh, Jin, Zhang, and Liu]{xiong2024buildingmathagentsmultiturn}
Wei Xiong, Chengshuai Shi, Jiaming Shen, Aviv Rosenberg, Zhen Qin, Daniele Calandriello, Misha Khalman, Rishabh Joshi, Bilal Piot, Mohammad Saleh, Chi Jin, Tong Zhang, and Tianqi Liu.
\newblock Building math agents with multi-turn iterative preference learning, 2024.
\newblock \url{https://arxiv.org/abs/2409.02392}.

\bibitem[Xu et~al.(2023)Xu, Lee, Sukhbaatar, and Weston]{xu2023some}
Jing Xu, Andrew Lee, Sainbayar Sukhbaatar, and Jason Weston.
\newblock Some things are more cringe than others: Preference optimization with the pairwise cringe loss.
\newblock \emph{arXiv preprint arXiv:2312.16682}, 2023.

\bibitem[Xu et~al.(2024)Xu, Wang, Wang, Lu, Xie, Saha, Sahoo, Yu, and Xiong]{xu2024aguvisunifiedpurevision}
Yiheng Xu, Zekun Wang, Junli Wang, Dunjie Lu, Tianbao Xie, Amrita Saha, Doyen Sahoo, Tao Yu, and Caiming Xiong.
\newblock Aguvis: Unified pure vision agents for autonomous gui interaction, 2024.
\newblock \url{https://arxiv.org/abs/2412.04454}.

\bibitem[Yang et~al.(2024)Yang, Yang, Hui, Zheng, Yu, Zhou, Li, Li, Liu, Huang, Dong, Wei, Lin, Tang, Wang, Yang, Tu, Zhang, Ma, Yang, Xu, Zhou, Bai, He, Lin, Dang, Lu, Chen, Yang, Li, Xue, Ni, Zhang, Wang, Peng, Men, Gao, Lin, Wang, Bai, Tan, Zhu, Li, Liu, Ge, Deng, Zhou, Ren, Zhang, Wei, Ren, Liu, Fan, Yao, Zhang, Wan, Chu, Liu, Cui, Zhang, Guo, and Fan]{yang2024qwen2technicalreport}
An~Yang, Baosong Yang, Binyuan Hui, Bo~Zheng, Bowen Yu, Chang Zhou, Chengpeng Li, Chengyuan Li, Dayiheng Liu, Fei Huang, Guanting Dong, Haoran Wei, Huan Lin, Jialong Tang, Jialin Wang, Jian Yang, Jianhong Tu, Jianwei Zhang, Jianxin Ma, Jianxin Yang, Jin Xu, Jingren Zhou, Jinze Bai, Jinzheng He, Junyang Lin, Kai Dang, Keming Lu, Keqin Chen, Kexin Yang, Mei Li, Mingfeng Xue, Na~Ni, Pei Zhang, Peng Wang, Ru~Peng, Rui Men, Ruize Gao, Runji Lin, Shijie Wang, Shuai Bai, Sinan Tan, Tianhang Zhu, Tianhao Li, Tianyu Liu, Wenbin Ge, Xiaodong Deng, Xiaohuan Zhou, Xingzhang Ren, Xinyu Zhang, Xipin Wei, Xuancheng Ren, Xuejing Liu, Yang Fan, Yang Yao, Yichang Zhang, Yu~Wan, Yunfei Chu, Yuqiong Liu, Zeyu Cui, Zhenru Zhang, Zhifang Guo, and Zhihao Fan.
\newblock Qwen2 technical report, 2024.
\newblock \url{https://arxiv.org/abs/2407.10671}.

\bibitem[Yao et~al.(2023)Yao, Chen, Yang, and Narasimhan]{yao2023webshopscalablerealworldweb}
Shunyu Yao, Howard Chen, John Yang, and Karthik Narasimhan.
\newblock Webshop: Towards scalable real-world web interaction with grounded language agents, 2023.
\newblock \url{https://arxiv.org/abs/2207.01206}.

\bibitem[Yuan et~al.(2024)Yuan, Li, Chen, Cui, Ding, Zhang, Zhou, Liu, and Peng]{yuan2024freeprocessrewardsprocess}
Lifan Yuan, Wendi Li, Huayu Chen, Ganqu Cui, Ning Ding, Kaiyan Zhang, Bowen Zhou, Zhiyuan Liu, and Hao Peng.
\newblock Free process rewards without process labels, 2024.
\newblock \url{https://arxiv.org/abs/2412.01981}.

\bibitem[Zhai et~al.(2024)Zhai, Bai, Lin, Pan, Tong, Zhou, Suhr, Xie, LeCun, Ma, and Levine]{zhai2024finetuninglargevisionlanguagemodels}
Yuexiang Zhai, Hao Bai, Zipeng Lin, Jiayi Pan, Shengbang Tong, Yifei Zhou, Alane Suhr, Saining Xie, Yann LeCun, Yi~Ma, and Sergey Levine.
\newblock Fine-tuning large vision-language models as decision-making agents via reinforcement learning, 2024.
\newblock \url{https://arxiv.org/abs/2405.10292}.

\bibitem[Zhou et~al.(2024{\natexlab{a}})Zhou, Xu, Zhu, Zhou, Lo, Sridhar, Cheng, Ou, Bisk, Fried, Alon, and Neubig]{zhou2024webarenarealisticwebenvironment}
Shuyan Zhou, Frank~F. Xu, Hao Zhu, Xuhui Zhou, Robert Lo, Abishek Sridhar, Xianyi Cheng, Tianyue Ou, Yonatan Bisk, Daniel Fried, Uri Alon, and Graham Neubig.
\newblock Webarena: A realistic web environment for building autonomous agents, 2024{\natexlab{a}}.
\newblock \url{https://arxiv.org/abs/2307.13854}.

\bibitem[Zhou et~al.(2024{\natexlab{b}})Zhou, Yang, Lin, Bai, Zhou, Wang, Levine, and Li]{zhou2024proposeragentevaluatorpaeautonomousskilldiscovery}
Yifei Zhou, Qianlan Yang, Kaixiang Lin, Min Bai, Xiong Zhou, Yu-Xiong Wang, Sergey Levine, and Erran Li.
\newblock Proposer-agent-evaluator(pae): Autonomous skill discovery for foundation model internet agents, 2024{\natexlab{b}}.
\newblock \url{https://arxiv.org/abs/2412.13194}.

\bibitem[Zhou et~al.(2024{\natexlab{c}})Zhou, Zanette, Pan, Levine, and Kumar]{zhou2024archertraininglanguagemodel}
Yifei Zhou, Andrea Zanette, Jiayi Pan, Sergey Levine, and Aviral Kumar.
\newblock Archer: Training language model agents via hierarchical multi-turn rl, 2024{\natexlab{c}}.
\newblock \url{https://arxiv.org/abs/2402.19446}.

\bibitem[Ziegler et~al.(2020)Ziegler, Stiennon, Wu, Brown, Radford, Amodei, Christiano, and Irving]{ziegler2020finetuninglanguagemodelshuman}
Daniel~M. Ziegler, Nisan Stiennon, Jeffrey Wu, Tom~B. Brown, Alec Radford, Dario Amodei, Paul Christiano, and Geoffrey Irving.
\newblock Fine-tuning language models from human preferences, 2020.
\newblock \url{https://arxiv.org/abs/1909.08593}.

\end{thebibliography}
\bibliographystyle{assets/plainnat}

\newpage
\appendix
\onecolumn
\section{Hyperparameters}
For completeness and reproducibility, we present all hyperparameters used in \methodname{} and all baselines in \autoref{tab:hyperparametersZZ}. In general, we found that the performances of both Multi-Turn DPO and SWEET-RL are consistent with respect to the hyperparameters in the DPO loss objective such as learning rate and beta, and the inclusion of a negative-log-likelihood coefficient of 0.01 helps in most cases (\citet{pang2024iterativereasoningpreferenceoptimization} also found this to be useful).

\begin{table}[h] 
\centering
\caption{Hyperparameters for \methodname{} and baseline methods for all experiments.}
\label{tab:hyperparametersZZ}
\resizebox{.8\linewidth}{!}{  
\begin{tabular}{c|c|cc} 
\toprule
& & Backend Programming& Frontend Design \\
\hline
\multirow{2}{8em}{Rejection Fine-Tuning} & learning rate & 2e-7 & 2e-7\\
& batch size& 32 & 32 \\
& epochs & 4 & 8 \\
\hline
\multirow{4}{8em}{Multi-Turn DPO} & learning rate & 2e-7 & 2e-7\\
& beta & 0.1 & 0.1\\
& negative-log-likelihood loss coefficient & 0.01 & 0.01\\
& batch size& 8 & 8 \\
& epochs & 4 & 8 \\
\hline
\multirow{8}{8em}{\methodname} & critic learning rate & 2e-7 & 2e-7\\
& critic beta & 0.1 & 0.1\\
& critic negative-log-likelihood loss coefficient & 0.01 & 0.01\\
& critic batch size& 8 & 8 \\
& critic epochs & 4 & 8 \\
& actor lr & 2e-7 & 2e-7\\
& actor beta & 0.1 & 0.1\\
& actor negative-log-likelihood loss coefficient & 0.01 & 0.01\\
& actor batch size& 8 & 8\\
& actor epochs & 1 & 1 
\end{tabular}}
\end{table}

\section{Theoretical Justifications} \label{app:theory}
In this section, we will provide some theoretical justifications for the derivation of our method. We first show that the trajectory-level Bradley-Terry objective can be similarly used for learning the advantage function:

\begin{lemma}\label{lem:rA}
    For an MDP with each state $(o,c) \in \mathcal{O} \times \mathcal{C}$ and for any $\pi: \mathcal{O} \mapsto \mathcal{A}$, assume the transition function $\mathcal{T}(o, a, c)$ is deterministic for any $(o, a, c)$ then the following holds $\forall \tau^+, \tau^-$:
\begin{align*}
    -\log \left[\sigma\left(\sum_{t} \beta r(o^+_t, a^+_t, c) - \sum_{t} \beta r(o^-_t, a^-_t, c)\right)\right] = -\log \left[\sigma\left(\sum_{t} \beta A^\pi(o^+_t, a^+_t, c) - \sum_{t} \beta A^\pi(o^-_t, a^-_t, c)\right)\right]
\end{align*}
\end{lemma}
\begin{proof}
    To show this, we would like to prove:
    \begin{align*}
        \sum_{t} r(o_t, a_t, c) = \sum_t A^\pi(o_t, a_t, c)
    \end{align*}
    We will prove this by telescoping:
    \begin{align*}
        & \sum_t A^\pi(o_t, a_t, c)\\
        = & \sum_t \left[Q^\pi(o_t, a_t, c) - V^\pi(o_t, c) \right]\\
        = & \sum_{t=1}^{N-1} \left[r(o_t, a_t, c) + \mathbb{E}_{o'_{t+1} \sim \mathcal{T}(\cdot|o_t, a_t, c)} V^\pi(o'_{t+1}, c) - V^\pi(o_t, c) \right]  + r(o_N, a_N, c)\\
        = & \sum_{t=1}^{N} r(o_t, a_t, c) + \sum_{t=1}^{N-1}\left[\mathbb{E}_{o'_{t+1} \sim \mathcal{T}(\cdot|o_t, a_t, c)} V^\pi(o'_{t+1}, c) - V^\pi(o_t, c) \right] \\
        = & \sum_{t=1}^{N} r(o_t, a_t, c),
    \end{align*}
    where the last equality follows by the assumption of deterministic transition.
\end{proof}

Additionally, we would like to also provide a theoretical justification for using an asymmetric critic for optimizing the actor with a different observation space. Intuitively, although for each sample such a training-time advantage function may give a different judgement compared to its regular counterpart, the following lemma shows that the policy gradient~\citep{Williams2004SimpleSG} estimated from an advantage function with training-time information is unbiased averaged over all samples. 

\begin{lemma}\label{lem:unbiased}
    For an MDP with each state $(o,c) \in \mathcal{O} \times \mathcal{C}$ and for any $\pi: \mathcal{O} \mapsto \mathcal{A}$, let $d^\pi_t(o_t, a_t, c)$ to be joint state-action occupancy distribution at step $t$, the following two estimators are both unbiased estimators of the policy gradient of $\pi$:
\begin{align*}
    \nabla \mathbb{E}_{\tau \sim \pi} \left(\sum_{t=1}^H r(o_t, a_t, c) \right) =& \sum_{t=1}^N \mathbb{E}_{o_t, a_t} A^\pi(o_t,a_t) \nabla \log \pi(a_t|o_t) 
    = \sum_{t=1}^N \mathbb{E}_{o_t,a_t, c} A^\pi(o_t,a_t, c) \nabla \log \pi(a_t|o_t)
\end{align*}
\end{lemma}

\begin{proof}
    The proof of this lemma is similar to the standard policy gradient analysis~\citep{agarwal2019reinforcement}.
    \begin{align*}
         \nabla \mathbb{E}_{\tau \sim \pi} \left(\sum_{t=1}^N r(o_t, a_t, c) \right) 
        = & \nabla \EE_{c} \mathbb{E}_{o_1} V^\pi(o_1)\\
        = & \EE_{c} \mathbb{E}_{o_1} \nabla V^\pi(o_1)\\
        = & \EE_{c} \mathbb{E}_{o_1} \nabla  \left[\sum_{a_1} \pi(a_1|o_1) Q^\pi(o_1, a_1)\right] \\
        = & \EE_{c} \mathbb{E}_{o_1} \left[\sum_{a_1} (\nabla \pi(a_1|o_1)) Q^\pi(o_1, a_1) + \sum_{a_1}\pi(a_1|o_1) \nabla Q^\pi(o_1, a_1) \right]\\
        = & \EE_{c} \mathbb{E}_{o_1} \left[\sum_{a_1} \pi(a_1|o_1) (\nabla \log\pi(a_1|o_1)) Q^\pi(o_1, a_1) + \sum_{a_1}\pi(a_1|o_1) \nabla Q^\pi(o_1, a_1) \right] \\
        = & \EE_{c} \mathbb{E}_{o_1} \left[\EE_{a_1} (\nabla \log\pi(a_1|o_1)) Q^\pi(o_1, a_1) + \sum_{a_1}\pi(a_1|o_1) \nabla Q^\pi(o_1, a_1) \right] \\
        = & \EE_{c} \mathbb{E}_{o_1} \left[\EE_{a_1} (\nabla \log\pi(a_1|o_1)) Q^\pi(o_1, a_1) + \EE_{a_1} \EE_{o_2}\nabla V^\pi(o_2) \right] \\
        = & \EE_{c} \sum_{t=1}^N \EE_{o_t, a_t} Q^\pi(o_t, a_t) \nabla \log \pi(a_t|o_t).
    \end{align*}
    To proceed, we need to first show a useful equality:
    \begin{align*}
        \EE_{o_t, a_t} V^\pi(o_t) \nabla \log \pi(a_t|o_t)
        = & \EE_{o_t} \sum_a \pi(a_t|o_t) V^\pi(o_t) \nabla \log \pi(a_t|o_t)\\
        = & \EE_{o_t} \sum_a V^\pi(o_t) \nabla \pi(a_t|o_t)\\
        = & \EE_{o_t}  V^\pi(o_t) \nabla \sum_a \pi(a_t|o_t)\\
        = & \EE_{o_t}  V^\pi(o_t) \nabla 1\\
        = &0
    \end{align*}
    Therefore, we can use the advantage function instead of the $Q$-function in the expression of policy gradients:
    \begin{align*}
         \nabla \mathbb{E}_{\tau \sim \pi} \left(\sum_{t=1}^N r(o_t, a_t, c) \right) 
         = & \EE_{c} \sum_{t=1}^N \EE_{o_t, a_t} Q^\pi(o_t, a_t) \nabla \log \pi(a_t|o_t)\\
         = & \EE_{c} \sum_{t=1}^N \EE_{o_t, a_t} (Q^\pi(o_t, a_t) - V^\pi(o_t, a_t)) \nabla \log \pi(a_t|o_t)\\
         = & \EE_{c} \sum_{t=1}^N \EE_{o_t, a_t} A^\pi(o_t, a_t) \nabla \log \pi(a_t|o_t)\\
         = & \sum_{t=1}^N \EE_{o_t, a_t}  \left(\nabla \log \pi(a_t|o_t)\right) \EE_{c}  A^\pi(o_t, a_t)\\
         = & \sum_{t=1}^N \EE_{o_t, a_t}  \left(\nabla \log \pi(a_t|o_t)\right) \EE_{c}  A^\pi(o_t, a_t, c)\\
         = & \EE_{c} \sum_{t=1}^N \EE_{o_t, a_t}  \left(\nabla \log \pi(a_t|o_t)\right)   A^\pi(o_t, a_t, c),
    \end{align*}
    where the second last equation follows from the fact that $\EE_{c \sim d^\pi_t(\cdot|o_t, a_t)} A^\pi(o_t, a_t, c) = A^\pi(o_t, a_t)$.
\end{proof}

\section{Prompts}
For completeness, we have included the prompt that we used for testing different models on Backend Programming in ~\autoref{fig:backend-prompt} and on Frontend Programming in ~\autoref{fig:frontend-prompt}.

\begin{figure}
    \centering
    \setlength{\fboxrule}{0.5pt}
    \fbox{
        \parbox{.95\textwidth}{
            \textbf{Backend Programming Prompt}\\
You are a helpful LLM agent. 

Your task is to help a human user to resolve their problem, in particular python programming.

1) Note that the problem is highly personalized so you need to explicitly gather information by asking questions to the human user about some hidden information and implicit constraints.

YOU SHOULD TRY TO ASK CLARIFICATION QUESTIONS.

2) Note that you should not ask human users complicated questions as they will only answer questions briefly in two sentences.

3) When you have gathered enough information to answer, say "I WANT TO ANSWER:" in the beginning of your response and provide your final answer.

4) Note that you can only interact with the human users WITHIN 10 back-and-forth rounds and you have to provide your final answer before the conversation ends.

5) You should be as concise as possible in your response to human.

"I WANT TO ANSWER:" should be included in your response to human if you think that you have gathered enough information for addressing this problem.

Directly output the raw python code after "I WANT TO ANSWER:".

Complete only the immediate agent response in this dialogue:
        }
    }
    \caption{\textbf{The prompt used for testing different models on Backend Programming task.}}
    \label{fig:backend-prompt}
\end{figure}

\begin{figure}
    \centering
    \setlength{\fboxrule}{0.5pt}
    \fbox{
        \parbox{.95\textwidth}{
            \textbf{Frontend Design Prompt}\\
You are a helpful LLM agent. 
Your task is to help a human user to code a complete website with a good design in HTML and Tailwind CSS.
Write the code inside a tag <html>.
Write real and long sentences about the business.
You don’t have to include images, but if you do, use only this source
https://picsum.photos/id/48/W/H, by replacing W and H with the width and height of the image.
Keep the id the same to only use id 48 image.

1) Note that the problem is highly personalized so you need to go through a few rounds of revisions.

2) When you have gathered enough information to answer, say "I WANT TO ANSWER:" in the beginning of your response and provide your final answer.

3) Note that you can only interact with the human users WITHIN 10 back-and-forth rounds and you have to provide your final answer before the conversation ends.

4) You will be judged both by the quality of the final answer and the efficiency of the conversation.

5) You can include ONLY ONE snippet raw html and Tailwind css code (wrapped in html tag)in your response to human user to ask how is the proposed design different from what the human user wants. 
This snippet of raw html and Tailwind css code (WRAPPED IN html TAG) will be rendered for the human to see a screenshot of the webpage. The human user will respond by comparing your rendered webpage with the webpage that the human user has in mind.

6) You need to make sure that your html webpage looks exactly as the human user wants, including the overall layout, navigation bars, background color etc.

7) The human user can only see your rendered image and provide suggestions based on the rendered image, and not any text questions.

First output your thought on your remaining uncertainties about the understanding of the problem and user preferences such as name of the function, input format, output format, and etc.
Then say "OUTPUT:\\n" followed by your proposal html.
        }
    }
    \caption{\textbf{The prompt used for testing different models on Frontend Design task.}}
    \label{fig:frontend-prompt}
\end{figure}
\section{Qualitative Comparisons of Different Credit Assignment Methods} \label{app:credit_assignment_qualitative}
We present qualitative comparison results of different credit assignment in \autoref{fig:qualitative}. First of all, we observe that LLM-as-a-Judge can easily get distracted by the length and format of the response without actually attending to its utility for task success. Furthermore, while being a natural practice in deep RL literature, the use of a value function fails to reasonably predit the expected future utility in unseen tasks. In particular, it predicts that the first candidate response has a probability of \emph{97\%} to lead to the final success of the agent despite being only the second turn out of 10 turns and this candidate response being phrased in a very confusing way. In contrast, \methodname{} is able to tell the advantage of the second response with a higher score because it is important in this task for the agent to figure out that the returned list can contain duplicate objects. This shows the importance of the RL algorithmic choices of \methodname{} and that the go-to practice of training a value function may generalize arbitrarily poorly in unseen tasks.

\begin{figure*}[!h]
     \centering
    \includegraphics[width=0.8\textwidth]{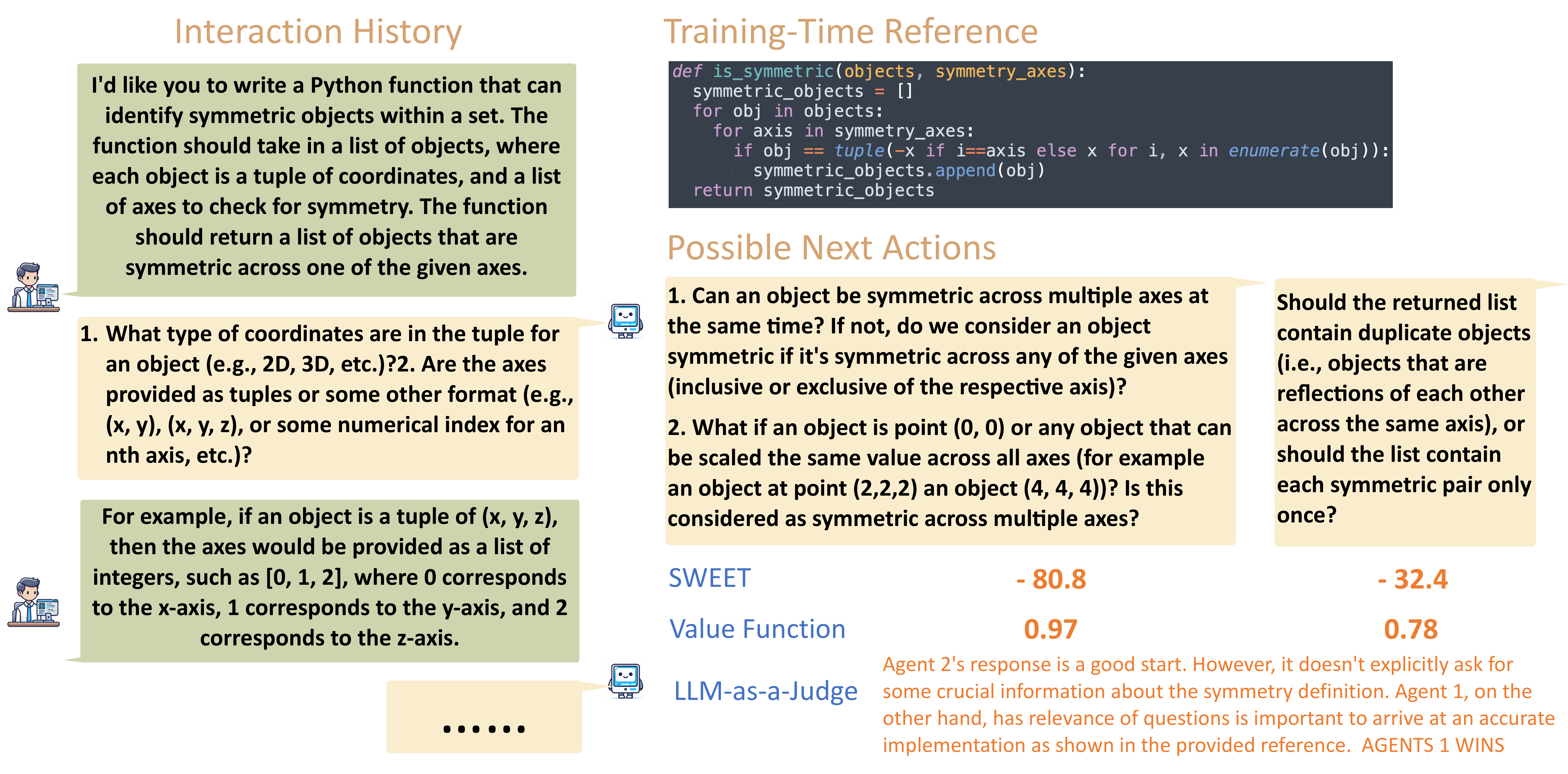}
        \caption{\textbf{Qualitative comparisons between different credit assignment methods.} A fixed LLM-as-a-Judge can be easily distracted by length and formats of the actions without considering their actual utility. A value function generalizes poorly to unseen tasks. In contrast, SWEET can attend to the actual utility of the action for task success and generalize well.}
        \label{fig:qualitative} 
        \vspace{-0.5cm}
\end{figure*}

\section{Full Qualitative Examples} \label{app:full_qualitative}
To demonstrate the level of difficulty of tasks in \benchmarkname{} and provide a qualitative comparisons of different models, we have included examples of full trajectories in this section. 

In particular, in \autoref{fig:backend_ours_qualiatative}, \autoref{fig:backend_ours_continued_qualiatative}, \autoref{fig:backend_zeroshot}, \autoref{fig:backend_gpt4o}, we have provided full trajectories on Backend Programming for \methodname{} Llama-3.1-8B-Instruct, Zeroshot Llama-3.1-8B-Instruct, and Zeroshot GPT4-O. While zeroshot baselines do try to propose some critical questions to seek more information from the human collaborator, they quickly jump into conclusions without gathering enough information, thus resulting in a wrong answer. Such failure modes exist even for stronger general-purpose LLMs like GPT4-O, indicating that task-specific tuning may always be necessary despite the improvement in the capability of the base model. In contrast, \methodname{} Llama-3.1-8B-Instruct learnt back-and-forth information-seeking behaviors and only answered the question once all information has been collected. Surprisingly, we found that RL training also results in some emergent behaviors such as reasoning with longer chain-of-thought and even self-corrections as shown in the last response from the agent in \autoref{fig:backend_ours_continued_qualiatative}.

We also include a full trajectory example on Frontend Design with \methodname{} Llama-3.1-8B-Instruct in \autoref{fig:frontend_1}, \autoref{fig:frontend_2}, \autoref{fig:frontend_3}, \autoref{fig:frontend_4}, \autoref{fig:frontend_5}, \autoref{fig:frontend_6}. We would like to note the significant complexity of this task where the agent needs to reason about HTML code over an extended horizon (up to 16k tokens), as a HTML code snippet is included in the response of each turn. After \methodname{} training, the LLM agent has learnt nuanced collaborative and reward-maximizing behaviors where it first proposes a scratch solution to gather coarse-grained feedback and only perform fine-grained edits in the end.

\begin{figure*}[!h]
     \centering
     \vspace{-0.3cm}
    \includegraphics[width=.95\textwidth]{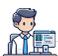}
        \caption{\footnotesize{\textbf{Example full trajectory for Backend Programming with \methodname{} Llama-3.1-8B-Instruct.} After training, the LLM agent has learnt back-and-forth information seeking behaviors before giving the final answer. }}
        \label{fig:backend_ours_qualiatative} 
\end{figure*}

\begin{figure*}[!h]
     \centering
     \vspace{-0.3cm}
    \includegraphics[width=.95\textwidth]{figures/backend_ours_continued.pdf}
        \caption{\footnotesize{\textbf{Example full trajectory for Backend Programming with \methodname{} Llama-3.1-8B-Instruct (Continued).} After training, the LLM agent has learnt back-and-forth information seeking behaviors before giving the final answer. }}        \label{fig:backend_ours_continued_qualiatative} 
\end{figure*}

\begin{figure*}[!h]
     \centering
    \includegraphics[width=.95\textwidth]{figures/backend_zeroshot.pdf}
        \caption{\footnotesize{\textbf{Example full trajectory for Backend Programming with Zeroshot Llama-3.1-8B-Instruct.} While the agent has asked a few questions, it quickly jumps into conclusions, resulting in a wrong final answer.}}
        \label{fig:backend_zeroshot} 
\end{figure*}

\begin{figure*}[!h]
     \centering
    \includegraphics[width=.95\textwidth]{figures/backend_gpt4o.pdf}
        \caption{\footnotesize{\textbf{Example full trajectory for Backend Programming with Zeroshot GPT-4O.} While the agent does propose critical questions to the human collaborator, it also has the issue of jumping into conclusions.}}
        \label{fig:backend_gpt4o} 
\end{figure*}

\begin{figure*}[!h]
     \centering
    \includegraphics[width=.95\textwidth]{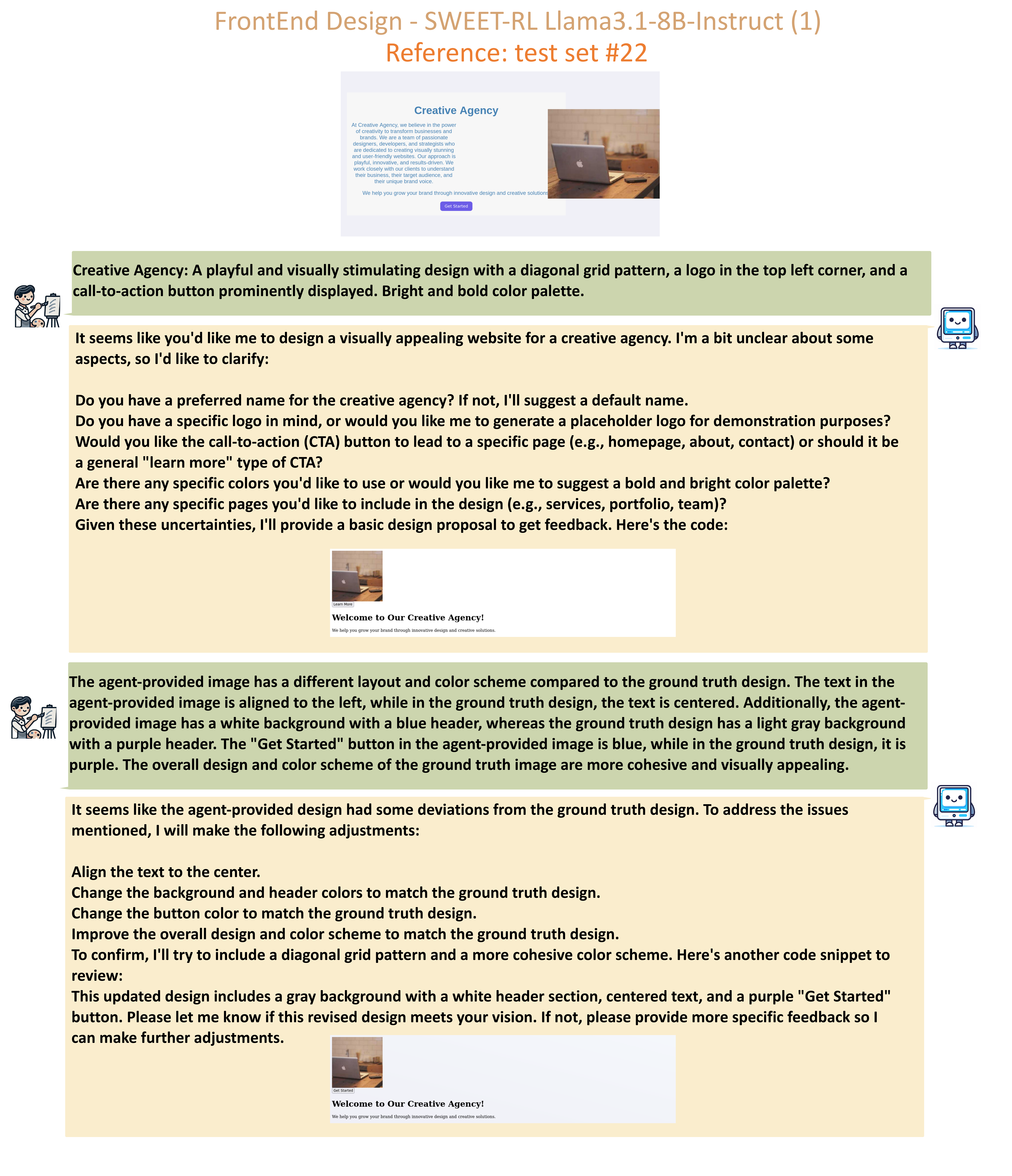}
        \caption{\footnotesize{\textbf{Example full trajectory for Frontend Design with \methodname{} Llama-3.1-8B-Instruct (1).} After training, the LLM agent has learnt sophisticated task-specific strategies to optimize the final reward.}}
        \label{fig:frontend_1} 
\end{figure*}

\begin{figure*}[!h]
     \centering
    \includegraphics[width=.95\textwidth]{figures/frontend_ours_2.pdf}
        \caption{\footnotesize{\textbf{Example full trajectory for Frontend Design with \methodname{} Llama-3.1-8B-Instruct (2).} After training, the LLM agent has learnt sophisticated task-specific strategies to optimize the final reward.}}
        \label{fig:frontend_2} 
\end{figure*}

\begin{figure*}[!h]
     \centering
    \includegraphics[width=.95\textwidth]{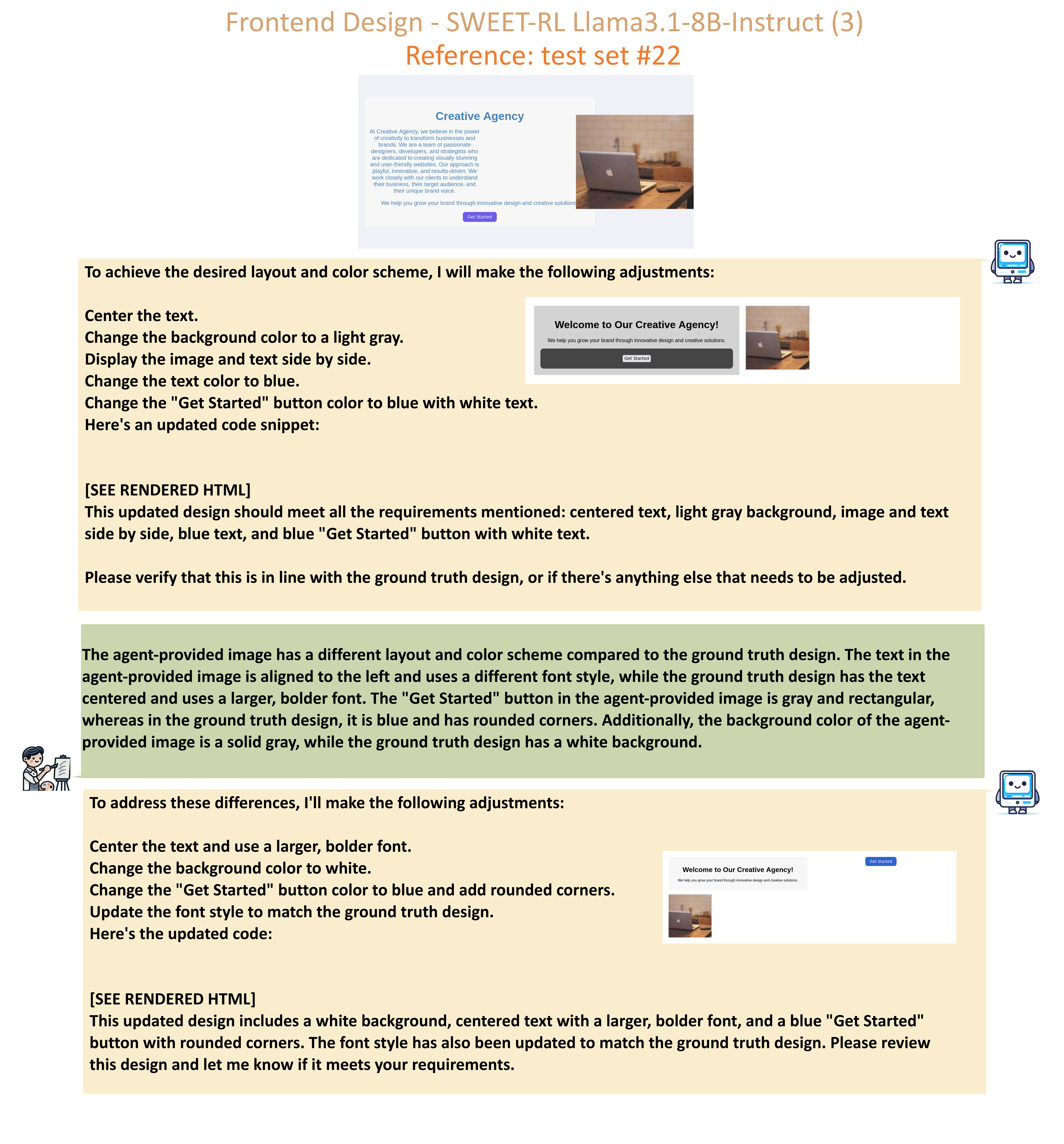}
        \caption{\footnotesize{\textbf{Example full trajectory for Frontend Design with \methodname{} Llama-3.1-8B-Instruct (3).} After training, the LLM agent has learnt sophisticated task-specific strategies to optimize the final reward.}}
        \label{fig:frontend_3} 
\end{figure*}

\begin{figure*}[!h]
     \centering
    \includegraphics[width=.95\textwidth]{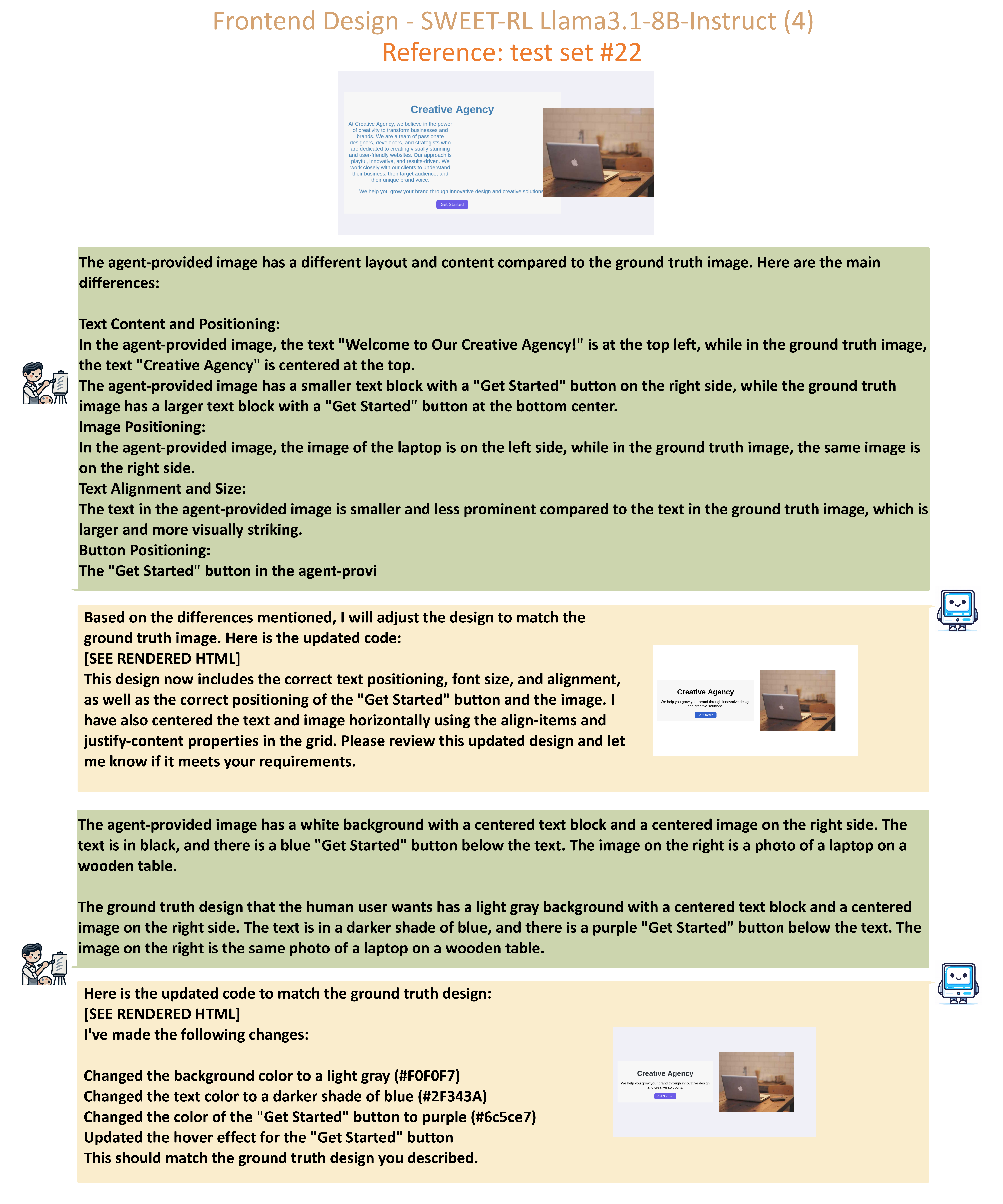}
        \caption{\footnotesize{\textbf{Example full trajectory for Frontend Design with \methodname{} Llama-3.1-8B-Instruct (4).} After training, the LLM agent has learnt sophisticated task-specific strategies to optimize the final reward.}}
        \label{fig:frontend_4} 
\end{figure*}

\begin{figure*}[!h]
     \centering
    \includegraphics[width=.95\textwidth]{figures/frontend_ours_5.pdf}
        \caption{\footnotesize{\textbf{Example full trajectory for Frontend Design with \methodname{} Llama-3.1-8B-Instruct (5).} After training, the LLM agent has learnt sophisticated task-specific strategies to optimize the final reward.}}
        \label{fig:frontend_5} 
\end{figure*}

\begin{figure*}[!h]
     \centering
    \includegraphics[width=.95\textwidth]{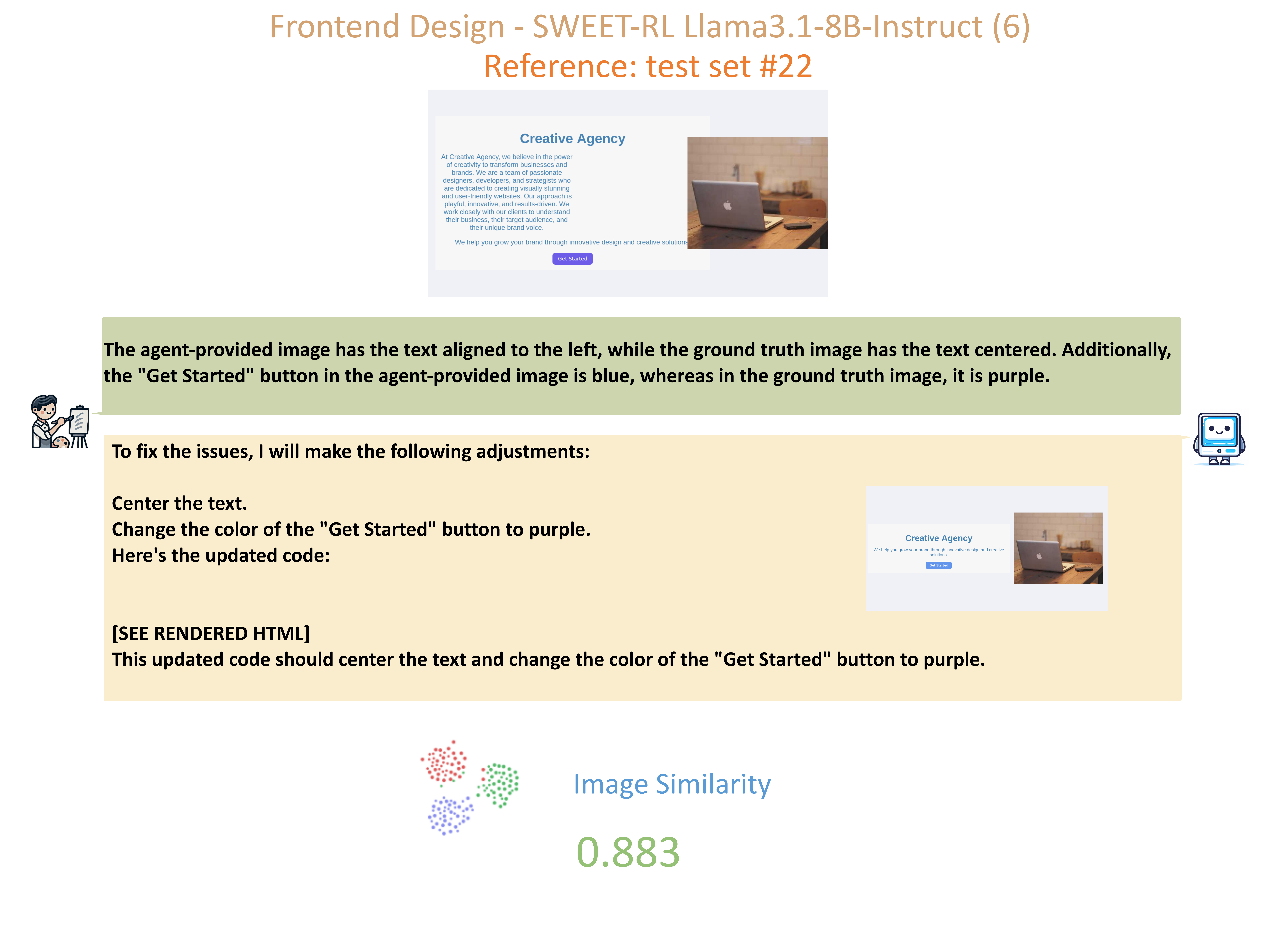}
        \caption{\footnotesize{\textbf{Example full trajectory for Frontend Design with \methodname{} Llama-3.1-8B-Instruct (6).} After training, the LLM agent has learnt sophisticated task-specific strategies to optimize the final reward.}}
        \label{fig:frontend_6} 
\end{figure*}

\end{document}